\documentclass{article}

\usepackage[main, final]{neurips_2026}
\usepackage[utf8]{inputenc} % allow utf-8 input
\usepackage[T1]{fontenc}    % use 8-bit T1 fonts
\usepackage{graphicx}
\usepackage{url}
\usepackage{booktabs}
\usepackage{mathtools}
\usepackage{amsfonts}
\usepackage{nicefrac}
\usepackage{microtype}
\usepackage{xcolor}
\usepackage{amsmath}
\usepackage{hyperref}
\usepackage{fontawesome5}
\usepackage{placeins}
\usepackage{enumitem}
\usepackage{wrapfig}
\usepackage[most]{tcolorbox}
\usepackage{xcolor}
\definecolor{lemmacolor}{HTML}{B0C8C2}
\definecolor{theoremcolor}{HTML}{85A196}
\usepackage{amsthm}
\newtheorem{lemma}{Lemma}
\newtheorem{theorem}{Theorem}
\tcolorboxenvironment{lemma}{
  colback=lemmacolor!18,
  colframe=lemmacolor!60,
  boxrule=0.3pt,
  arc=1pt,
  left=5pt,
  right=5pt,
  top=5pt,
  bottom=5pt,
  before skip=4pt,
  after skip=4pt,
  breakable,
  enhanced jigsaw
}
\tcolorboxenvironment{theorem}{
  colback=theoremcolor!15,
  colframe=theoremcolor!65,
  boxrule=0.3pt,
  arc=1pt,
  left=5pt,
  right=5pt,
  top=5pt,
  bottom=5pt,
  before skip=4pt,
  after skip=4pt,
  breakable,
  enhanced jigsaw
}
\usepackage[table]{xcolor}
\definecolor{MyBlue}{HTML}{E1D5E7}

\definecolor{MyText}{HTML}{6989BA}
\newcommand{\parenclrbf}[1]{{\small\bfseries\textcolor{MyText}{#1}}}
\usepackage{xcolor}
\definecolor{re}{HTML}{8630f4}

\title{Test-Time Prompt-Agnostic Decomposition}

\author{Junze Wang$^{1}$
Lei Fan$^{2}$
Dezheng Zhang$^{1}$
Donglin Di$^{3}$ 
Yang Song$^{2}$
Sidong Liu$^{4}$
Cong Cong$^{4}$\thanks{Corresponding author.} \\
$^{1}$University of Science and Technology
Beijing 
$^{2}$University of New South Wales \\
$^{3}$Sympflux Co., Ltd.
$^{4}$Macquarie University\\
\texttt{wangjunze@yeah.net,thomas.cong@mq.edu.au}}

\begin{document}

\maketitle
\setcounter{footnote}{0}

\begin{abstract}
Test-Time Prompt Tuning (TPT) adapts large pretrained models to unseen target distribution shifts by updating only lightweight prompt tokens. Existing TPT methods are effective, but they mainly optimize each prompt update locally. Across test-time steps, prompts updated at the current step are reused for the corresponding predictions and updates, so noisy or biased unlabeled objectives can accumulate and destabilize the prompt-update trajectory. We identify two failure modes: magnitude expansion, where updates move the prompt too far, and directional drift, where consecutive updates point in inconsistent directions. To analyze and control these failures, we propose Test-Time Prompt-Agnostic Decomposition ($\mathtt{TPD}$), which characterizes observed prompt-update trajectories from a dynamical-systems view and decomposes base prompt updates into magnitude and direction. $\mathtt{TPD}$ computes a spectral radius to measure update expansion and decomposes update directions into persistent, oscillatory, and residual components. Building on this decomposition, Adaptive Koopman Control ($\mathtt{AKC}$) regulates update magnitude by shrinking updates under expansive recent dynamics, while Hankel Update Router ($\mathtt{HUR}$) refines update direction by preserving persistent components and suppressing oscillatory and residual components. Together, $\mathtt{AKC}$ and $\mathtt{HUR}$ produce a stabilized prompt update. As a prompt-agnostic framework, $\mathtt{TPD}$ can be plugged into visual and text TPT methods without modifying the backbone, prompt architecture, or adaptation loss. Experiments on 15 datasets show that $\mathtt{TPD}$ consistently improves 10 baselines, with accuracy gains of $3$--$8\%$ for visual prompts and over $2\%$ for text prompts\protect\footnotemark.
\end{abstract}
\footnotetext{The code is available at \href{https://github.com/reeive/TPD}{https://github.com/reeive/TPD}.}

\section{Introduction}

Test-Time Prompt Tuning (TPT) has emerged as a parameter-efficient paradigm for adapting large pretrained vision and vision-language models, such as ViT~\citep{dosovitskiy2020image} and CLIP~\citep{radford2021learning}, to unseen target distribution shifts~\citep{yuan2023robust}. Instead of updating the full backbone, TPT adapts only prompt tokens during inference, making it well suited for frozen foundation models~\citep{liang2025comprehensive}.
Existing methods adapt prompts from unlabeled test data using step-level objectives, including entropy minimization~\citep{shu2022test}, confidence-based augmentation selection~\citep{feng2023diverse, farina2024frustratingly}, prompt calibration~\citep{yoon2024c}, prompt-history regularization~\citep{ma2023swapprompt,xiao2025dynaprompt}, and black-box prompt updates~\citep{meng2025black}. Although effective, these methods primarily optimize each prompt update locally. This local-update view is incomplete for TPT across steps: prompts updated at the current step are used for the corresponding predictions and updates, so noisy or biased updates from unlabeled objectives can influence later prompt states and accumulate over time~\citep{niu2023towards}.

This cross-step dependency exposes a key weakness of existing TPT methods. Consecutive updates may move with inappropriate magnitudes or point in inconsistent directions, leading to prompt-update trajectory instability. As illustrated in Figure~\ref{fig1}(a), CLIP-based text TPT shows pronounced oscillations in accuracy and update behavior, while ViT-based visual TPT exhibits unstable update behavior together with persistently low accuracy across test steps. These observations indicate that existing methods lack a principled tool for analyzing and controlling prompt updates across test steps. To better characterize prompt-update trajectories, we decompose each update into two complementary components: magnitude, which describes the update step size, and direction, which describes the update orientation \citep{bottou2018optimization}. This magnitude-direction view provides a structured way to analyze prompt-update trajectories: instability may arise when update magnitudes expand across steps, when update directions drift or oscillate over time, or when both occur jointly.

Accordingly, we propose Test-Time Prompt-Agnostic Decomposition ($\mathtt{TPD}$), which characterizes observed prompt-update trajectories from a dynamical-systems perspective~\citep{williams2015data}. $\mathtt{TPD}$ operationalizes the magnitude–direction view through the prompt-update decomposition, enabling both analysis and stabilization across test steps. As illustrated in Figure~\ref{fig1}(b), $\mathtt{TPD}$ provides a quantitative interpretation of the unstable behaviors observed in Figure~\ref{fig1}(a) by decomposing trajectory dynamics into two complementary branches. For the magnitude branch, $\mathtt{TPD}$ computes a spectral radius as an expansion measure, where larger values indicate stronger growth of update magnitudes across recent steps. For the direction branch, $\mathtt{TPD}$ decomposes each update into three additive components: a persistent component capturing consistent adaptation signals, an oscillatory component capturing sign-flipping interference, and a residual component capturing unexplained noise. These quantities provide an explicit basis for magnitude expansion and directional drift.

\begin{figure}[t]
    \centering
    \includegraphics[width=0.95\linewidth]{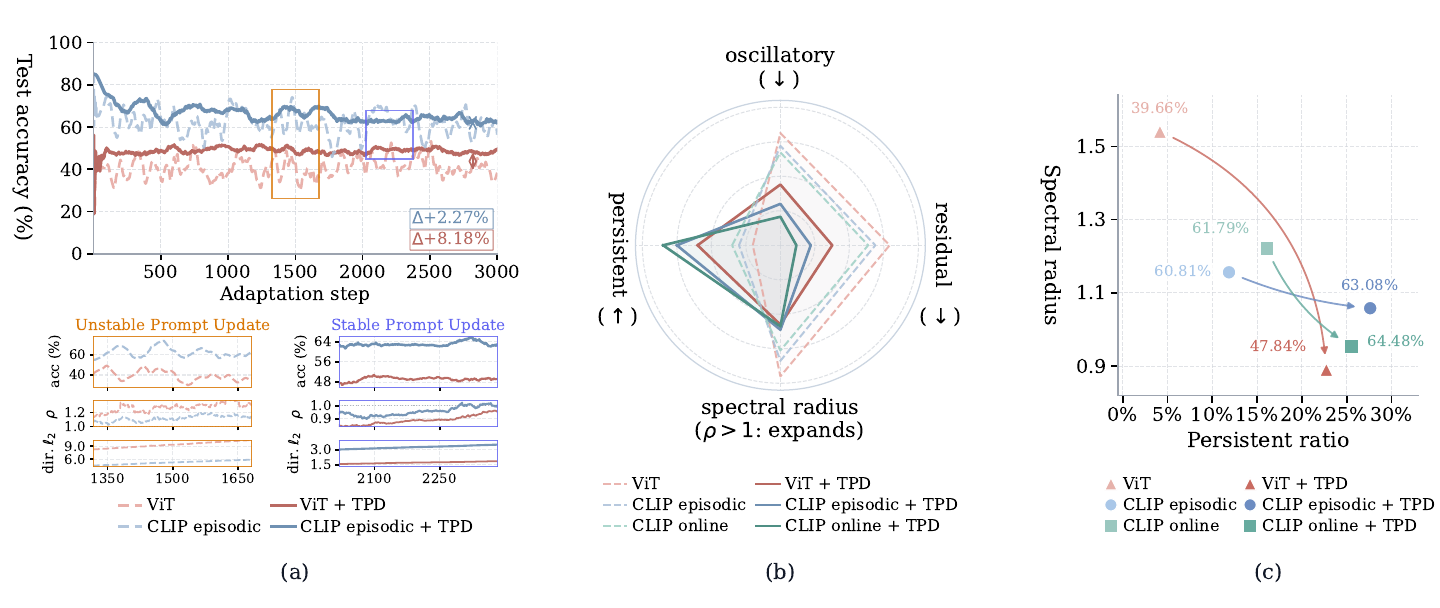}
    \caption{(a) The accuracy of ViT visual and CLIP text prompts in episodic modes, averaged over four ImageNet variants. Gains denote $\mathtt{TPD}$ improvements. 
    The zoomed windows show accuracy, magnitude (spectral radius $\rho$), and direction ($\ell_2$ distance between consecutive update vectors, decomposed in (b)) in unstable and stable regimes. (b) $\mathtt{TPD}$ decomposes each prompt update into persistent, oscillatory, and residual components, with the spectral radius measuring local expansion. (c) Stability-performance relation summarized by spectral radius and persistent ratio.}
    \label{fig1}
\end{figure}

The decomposition quantities reveal a stability pattern in Figure~\ref{fig1}(c): methods with a large spectral radius and a low persistent ratio consistently exhibit poor performance. Specifically, the spectral radius summarizes magnitude expansion, while the persistent ratio denotes the fraction of update direction assigned to the persistent component.
To this end, we propose two complementary mechanisms. Adaptive Koopman Control ($\mathtt{AKC}$) regulates update magnitude by shrinking updates when recent prompt dynamics indicate expansion. Hankel Update Router ($\mathtt{HUR}$) refines update direction by preserving persistent components across recent updates while suppressing oscillatory and residual components. These two mechanisms produce a stabilized prompt update. Since $\mathtt{TPD}$ only processes the base prompt update, it can be plugged into both visual and text TPT methods without modifying the backbone, prompt architecture, or adaptation loss. The contributions are summarized as follows:
\begin{itemize}
\item To the best of our knowledge, Test-Time Prompt-Agnostic Decomposition ($\mathtt{TPD}$) is the first prompt-agnostic framework that analyzes and stabilizes prompt-update trajectories across ViT-based visual TPT and CLIP-based text TPT methods. 
\item We introduce Adaptive Koopman Control ($\mathtt{AKC}$), which uses the spectral radius to scale down updates under expansive local dynamics and further attenuates them under consecutive unstable steps, preventing excessive prompt-state movement.
\item We propose Hankel Update Router ($\mathtt{HUR}$), which uses the three-component decomposition to reshape update directions by preserving persistent components and suppressing oscillatory and residual components, thereby mitigating directional conflict and cumulative drift.
\item Extensive experiments on 15 datasets across shifted-domain and cross-dataset benchmarks demonstrate the generality of $\mathtt{TPD}$. Across 10 TPT methods, $\mathtt{TPD}$ achieves the accuracy gains of $3$--$8\%$ for visual prompts and over $2\%$ for text prompts.
\end{itemize}

\section{Related Work}

\noindent\textbf{Prompt-based test-time adaptation.}
Most existing TPT methods are developed for CLIP-based text prompt tuning and can be broadly divided into episodic and online modes. In episodic mode, prompts are adapted independently for each test sample or episode and are typically reset afterward, including initialization and calibration~\citep{abdul2023align, yoon2024c}, sample filtering or augmentation selection~\citep{shu2022test, feng2023diverse, farina2024frustratingly}, and black-box prompt update rules~\citep{meng2025black}. In online mode, prompts are updated continuously along test steps, and existing methods mainly regulate prompt drift or leverage prompt history~\citep{ma2023swapprompt, zhang2024historical, xiao2025dynaprompt}. Furthermore, CTPT~\citep{wang2025ctpt} focuses on continual prompt tuning across multiple target domains, aiming at prompt retention and generalization over newly learned streaming data. Some prompt-based methods are also built on ViT-based visual prompts~\citep{jia2022visual}; for example, DePT~\citep{gao2022visual} improves ViT prompt adaptation under domain adaptation with memory-bank and pseudo-label.

Despite these differences, existing TPT methods still intervene mainly at individual test steps. Moreover, CTPT~\citep{wang2025ctpt} targets continual multi-domain prompt retention, and DePT~\citep{gao2022visual} is developed for domain adaptation. In contrast, our method analyzes and stabilizes cross-step prompt-update trajectories within each single target-domain adaptation process.

\noindent\textbf{Prompt-based tuning and optimization.}
Prompt-based optimization has been developed across downstream tasks. Specifically, some methods improve prompt construction and backbone interaction \citep{xiao2026layer}. VPT introduces learnable visual prompts for adapting a frozen ViT \citep{jia2022visual}; E2VPT extends visual prompt tuning with more effective prompt insertion across layers \citep{han20232vpt}; and CVPT further strengthens prompt interaction across layers \citep{huang2025cvpt}. In vision-language models, ArGue improves prompt construction through attribute-guided prompts \citep{tian2024argue}, while TCP uses class-aware textual prompts to better align prompts with class semantics \citep{yao2024tcp}. Feature prompting \citep{cong2026slide,cong2026flexible} and hierarchical image--text learning \citep{fan2026medical} have also been explored in medical imaging applications. Furthermore, some methods improve optimization and stabilization during training \citep{xiao2026staying,xiao2026not,lan2024gradient,lan2026trust}. Fair-VPT introduces fairness-aware objectives to regularize prompt learning \citep{park2024fair}, BVPT analyzes bursty update behavior during prompt optimization \citep{wang2025attention}.
PAE \citep{wang2026visual} is related to $\mathtt{TPD}$ in its use of dynamical systems, but it focuses on training-time visual prompt optimization by modeling cross-layer prompt evolution. In contrast, $\mathtt{TPD}$ analyzes and stabilizes the prompt-update trajectory of test samples. 

These methods improve prompt learning through more effective training-time optimization, but prompt optimization under test-time inference remains largely unexplored. In contrast, our work focuses on prompt adaptation during test-time inference, with an emphasis on trajectory-level diagnosis and control of prompt updates across test steps.

\section{Method}

% \subsection{Preliminary}
\subsection{Test-Time Prompt-Agnostic Decomposition}

Given unlabeled test inputs, let $f(\mathcal{B}_t;\theta,\mathcal{P}_t)$ denote a pretrained foundation model at test step $t$ with frozen backbone $\theta$, \textit{e.g.}, ViT and CLIP. For ViT-based visual TPT, $\mathcal{P}_t$ consists of learnable prompt tokens injected into the encoder blocks, and $\mathcal{B}_t$ is the current mini-batch across test steps. For CLIP-based text TPT, $\mathcal{P}_t$ consists of learnable text prompt tokens used to construct class textual features, while $\mathcal{B}_t$ depends on the test-time adaptation (TTA) mode: in online mode, $\mathcal{B}_t$ is the current input that updates the text prompts cumulatively; in episodic mode, $\mathcal{B}_t$ is the current episode, typically a set of augmented views from a single test sample. 

For each ViT- or CLIP-based TPT method, we flatten its prompt tokens into a single vector:
\begin{equation}
\mathbf{p}_t = \operatorname{vec}(\mathcal{P}_t) \in \mathbb{R}^{d},
\label{eq1}
\end{equation}
where $\operatorname{vec}$ concatenates the prompt tokens, and $d$ is the prompt-vector dimension. Given $\mathcal{B}_t$ and a TTA loss $\mathcal{L}_{\mathrm{tta}}$, which is typically employed as entropy minimization~\citep{wang2020tent}, the base prompt-update $\mathbf{u}_t$ at step~$t$ is
\begin{equation}
\mathbf{u}_t
= -\eta_t \phi_t (\nabla_{\mathbf{p}} \mathcal{L}_{\mathrm{tta}}(\mathcal{B}_t; \mathbf{p}_t)).
\label{eq2}
\end{equation}
where $\eta_t$ is the learning rate at step~$t$, $\phi_t$ denotes the base optimizer, and $\nabla_{\mathbf{p}} \mathcal{L}_{\mathrm{tta}}(\mathcal{B}_t; \mathbf{p}_t)$ denotes the gradient of the TTA loss with respect to the prompt state $\mathbf{p}_t$. 

By taking the base prompt-update $\mathbf u_t$ as input and stabilizing it directly, $\mathtt{TPD}$ is agnostic to visual and textual prompt types. As shown in Figure \ref{fig2}, $\mathtt{TPD}$ analyzes the recent prompt-update trajectory in a short-window subspace and decomposes it: the spectral radius captures magnitude expansion, while an additive decomposition separates persistent, oscillatory, and residual directional components.

\begin{figure}[t]
    \centering
    \includegraphics[width=\linewidth]{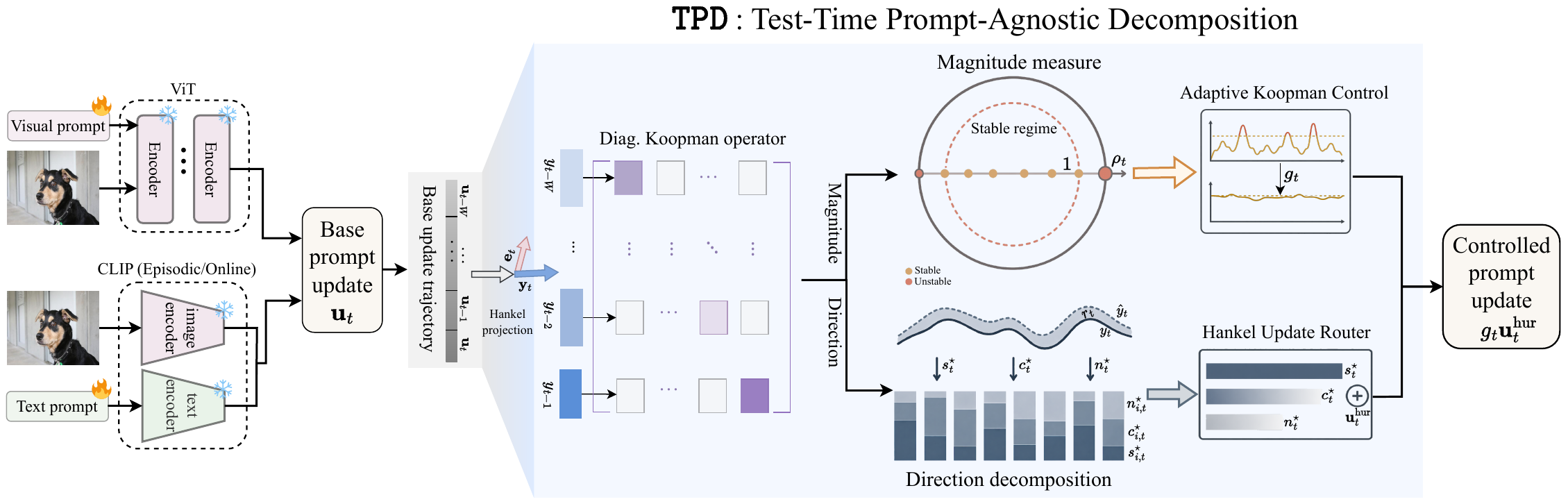}
    \caption{Overview of $\mathtt{TPD}$. The base prompt update $\mathbf{u}_t$ from visual or text TPT is delay-embedded with a Hankel window and projected into a local Hankel subspace. $\mathtt{AKC}$ regulates update magnitude and $\mathtt{HUR}$ refines update direction. The controlled magnitude and refined direction are then combined into a stabilized prompt update $g_t\mathbf u_t^{\mathrm{hur}}$, which is applied to the prompt state at step $t$.}
    \label{fig2}
\end{figure}

\noindent\textbf{Hankel subspace from recent updates.}
$\mathtt{TPD}$ builds a short-window representation of the recent prompt-update trajectory using the delay-coordinate construction in Hankel-DMD~\citep{arbabi2017ergodic}. Specifically, it uses the latest $W$ base prompt updates and converts them into $M$ overlapping Hankel embeddings, where each embedding stacks $K$ consecutive updates and $W=K+M-1$. For any index $\tau\ge K$, the Hankel-lifted vector $\mathbf z_\tau$ stacks the latest $K$ update vectors: $\mathbf{z}_{\tau}=\big[
\mathbf{u}_{\tau-K+1}^{\top},
\mathbf{u}_{\tau-K+2}^{\top},
\dots,
\mathbf{u}_{\tau}^{\top}
\big]^{\top}
\in\mathbb{R}^{Kd}$.
Using the $M$ embeddings before the current update, we define the Hankel history matrix at test step $t$:
\begin{equation}
\mathcal{H}_t
=
\big[
\mathbf{z}_{t-M},
\mathbf{z}_{t-M+1},
\dots,
\mathbf{z}_{t-1}
\big]
\in\mathbb{R}^{Kd\times M}.
\label{eq:main_history}
\end{equation}
Let $Q_t\in\mathbb{R}^{Kd\times r}$ contain the top-$r$ left singular vectors of $\mathcal H_t$\footnote{By the Eckart--Young theorem~\citep{eckart1936approximation}, the top singular directions provide the optimal rank-$r$ subspace.}, where $r\le \min(Kd,M)$ is selected adaptively from the current singular spectrum, and project the current lifted update $\mathbf z_t$ onto this short-window Hankel subspace:
\begin{equation}
\mathbf{y}_t=Q_t^\top\mathbf{z}_t\in\mathbb{R}^{r},
\qquad
\mathbf{z}_t=Q_t\mathbf{y}_t+\widetilde{\mathbf{e}}_t,
\qquad
\widetilde{\mathbf{e}}_t=(I_{Kd}-Q_tQ_t^\top)\mathbf{z}_t.
\label{eq:main_projection_lifted}
\end{equation}
The prompt-update decomposition is obtained by extracting the most recent block:
\begin{equation}
\mathbf{u}_t=E_KQ_t\mathbf{y}_t+\mathbf{e}_t,
\qquad
\mathbf{e}_t=E_K\widetilde{\mathbf{e}}_t,
\qquad
E_K=[\mathbf{0}_{d\times (K-1)d}\;\;I_d],
% \in\mathbb{R}^{d\times Kd},
\label{eq:main_projection_current}
\end{equation}
where $\widetilde{\mathbf e}_t$ is the residual, and $\mathbf e_t$ is its most recent $d$-dimensional prompt-update block.

\noindent\textbf{Short-window Koopman operator.}
To capture recent update dynamics in the projected space, we map the historical lifted updates in $\mathcal H_t$ with the projection matrix $Q_t$ and construct
\begin{equation}
\begin{aligned}
Y_t^-&=\left[Q_t^\top z_{t-M},\ldots,Q_t^\top z_{t-2}\right]
=\left[\mathbf y_{t-M},\ldots,\mathbf y_{t-2}\right],\\
Y_t^+&=\left[Q_t^\top z_{t-M+1},\ldots,Q_t^\top z_{t-1}\right]
=\left[\mathbf y_{t-M+1},\ldots,\mathbf y_{t-1}\right].
\end{aligned}
\end{equation}
Each column pair thus represents one recent update transition before test step~$t$. We then fit a diagonal operator $A_t=\operatorname{diag}(\mathbf a_t)$ using scalar least-squares estimates:
\begin{equation}
a_{i,t}=\frac{\sum_{\tau=t-M}^{t-2}y_{i,\tau+1}y_{i,\tau}}
{\sum_{\tau=t-M}^{t-2}y_{i,\tau}^{2}},
\qquad i=1,\ldots,r,
\qquad \mathbf a_t=[a_{1,t},\ldots,a_{r,t}]^\top,
\label{eq:main_koopman}
\end{equation}
where $y_{i,\tau}$ is the $i$-th scalar entry of $\mathbf y_\tau$. The sign of $a_{i,t}$ distinguishes continuation from reversal, while its magnitude measures amplification in that coordinate.

\paragraph{Spectral radius: a magnitude signal.}
To quantify the magnitude expansion across steps, we compute the spectral radius $\rho_t=\rho(A_t)=\max |a_{t}|$. The larger $\rho_t$ indicates stronger local magnitude expansion at step~$t$. Since $A_t$ is estimated only from projected lifted updates before step~$t$, $\rho_t$ uses no future information and can be monitored online.

\paragraph{Three-component decomposition: a direction signal.}
To analyze the direction of each prompt update, $\mathtt{TPD}$ leverages Dynamic Mode Decomposition (DMD) \citep{schmid2010dynamic} under the local Koopman model and decomposes each projected update into three additive components. Let $\hat{\mathbf y}_t=A_t\mathbf y_{t-1}$ be the prediction and  $\mathbf r_t=\mathbf y_t-\hat{\mathbf y}_t$ be the prediction residual. From the Koopman operator $A_t$ and its prediction $\hat{\mathbf y}_t$, we derive two element-wise gates
\begin{equation}
\mathbf q_t =
\frac{\hat{\mathbf y}_t^{\,2}}
{\hat{\mathbf y}_t^{\,2}+\mathbf r_t^{\,2}},
\qquad
\boldsymbol\pi_t =
\frac{1+\tanh(\mathbf a_t)}{2},
\label{eq:main_ratios}
\end{equation}
where $\mathbf q_t\in[0,1]^r$ measures how well the local Koopman model explains $\mathbf y_t$, and $\boldsymbol\pi_t\in[0,1]^r$ distinguishes same-sign continuation from sign flipping, recovering the masks $\mathbb{1}[\mathbf a_t>0]$ and $\mathbb{1}[\mathbf a_t<0]$. Using these gates, the projected update admits the decomposition
\begin{equation}
\mathbf y_t
=
\underbrace{\mathbf q_t\boldsymbol\pi_t\mathbf y_t}_{\mathbf s_t^\star}
+
\underbrace{\mathbf q_t(1-\boldsymbol\pi_t)\mathbf y_t}_{\mathbf c_t^\star}
+
\underbrace{(1-\mathbf q_t)\mathbf y_t}_{\mathbf n_t^\star}.
\label{eq:main_components} 
\end{equation}
The persistent component $s_t^\star$ receives weight through both local predictability and same-sign continuation. The oscillatory component $c_t^\star$ captures locally explained reversal, while $n_t^\star$ contains content unexplained by the fitted dynamics. These are operational descriptions of the recent update trajectory. Oscillatory and residual components may still contain useful corrections, which motivates retaining them with nonzero routing weights (see Appendix~\ref{app:spectral_decomposition} for justification).

\subsection{Adaptive Koopman Control}

Given the magnitude signal $\rho_t$ from $\mathtt{TPD}$, $\mathtt{AKC}$ modulates the prompt-update magnitude through a multiplicative gain factor $g_t$ applied at each step:
\begin{equation}
g_t=
\begin{cases}
1, & \rho_t \le \bar\rho,\\[2pt]
\max\!\left(g_{\min},\,\dfrac{\bar\rho}{\rho_t}\right), & \rho_t > \bar\rho,
\end{cases}
\label{eq:main_gain}
\end{equation}
where $\bar\rho$ is a stability threshold and $g_{\min}\in(0,1)$ is a lower bound on the gain. $\mathtt{AKC}$ preserves the update when $\rho_t\le\bar\rho$ and scales it down by $g_t<1$ when $\rho_t>\bar\rho$. Across consecutive unstable steps, this scaling keeps the applied updates small while the local dynamics remain unstable (see Lemma~\ref{lem:diag_vs_full} for the comparison between diagonal and full modes).

\begin{lemma}[$\mathtt{AKC}$ non-amplification]
\label{lem:akc}
Let $A_t$ be the local Koopman operator in Eq.~\eqref{eq:main_koopman} and $g_t$ the gain factor in Eq.~\eqref{eq:main_gain}. Then $\rho(g_tA_t)\;=\;g_t\,\rho(A_t)\;\le\;\rho(A_t)$, with equality if and only if $\rho_t\le\bar\rho$. Moreover,
\begin{equation}
\rho(g_tA_t)\;\le\;\max\bigl(\bar\rho,\;g_{\min}\rho_t\bigr),
\label{eq:lem_akc_upper}
\end{equation}
and in particular $\rho(g_tA_t)\le\bar\rho$ whenever $\rho_t\le\bar\rho/g_{\min}$.

\noindent\textit{\textbf{Proof sketch.}}
$g_t$ scales the eigenvalues of $A_t$, and the remaining bounds follow by splitting the gain definition in Eq.~\eqref{eq:main_gain}. The gain floor prevents update shutdown, but it also means that the controlled spectral radius is not always below $\bar\rho$; the applicable bound is Eq.~\eqref{eq:lem_akc_upper}.
\end{lemma}

\subsection{Hankel Update Router}

Given the $\mathtt{TPD}$ three-component decomposition, $\mathtt{HUR}$ refines the update direction in the Hankel-lifted space and then extracts the most recent prompt-update block. Let $\mathbf s_t^\star,\mathbf c_t^\star,\mathbf n_t^\star\in\mathbb R^r$ collect the coordinate-wise persistent, oscillatory, and unexplained components. Let $\widetilde{\mathbf e}_t=(I_{Kd}-Q_tQ_t^\top)\mathbf z_t$ be the residual in the Hankel-lifted space and $\mathbf e_t=E_K\widetilde{\mathbf e}_t$ its extracted prompt-update block. We define the lifted components:
\begin{equation}
\widetilde{\mathbf S}_t=Q_t\mathbf s_t^\star,\qquad
\widetilde{\mathbf C}_t=Q_t\mathbf c_t^\star,\qquad
\widetilde{\mathbf N}_t=Q_t\mathbf n_t^\star+\widetilde{\mathbf e}_t.
\label{eq:main_lifted_components}
\end{equation}
We treat both the projected unexplained component $\mathbf n_t^\star$ and the residual $\widetilde{\mathbf e}_t$ as residual content. Their prompt-update-space counterparts are
\begin{equation}
\mathbf S_t=E_K\widetilde{\mathbf S}_t,\qquad
\mathbf C_t=E_K\widetilde{\mathbf C}_t,\qquad
\mathbf N_t=E_K\widetilde{\mathbf N}_t
=E_KQ_t\mathbf n_t^\star+\mathbf e_t .
\label{eq:main_lift}
\end{equation}
Thus $\mathbf u_t=\mathbf S_t+\mathbf C_t+\mathbf N_t$, where $\mathbf S_t,\mathbf C_t,\mathbf N_t\in\mathbb R^d$ are additive components induced by the projected Hankel decomposition. $\mathtt{HUR}$ routes the prompt-space components by preserving the persistent component and attenuating the oscillatory and residual components:
\begin{equation}
\mathbf u_t^{\mathrm{hur}}
=
\mathbf S_t+\beta_c\mathbf C_t+\beta_n\mathbf N_t,
\qquad
0\le\beta_n\le\beta_c\le1,
\label{eq:main_hur}
\end{equation}
where $\beta_c$ controls the oscillatory component and $\beta_n$ controls the residual component. Let $\widetilde{\mathbf z}^{\mathrm{hur}}_t$ denote the corresponding lifted routed update, obtained by applying the same routing in the Hankel-lifted space before extracting $\mathbf u_t^{\mathrm{hur}}=E_K\widetilde{\mathbf z}^{\mathrm{hur}}_t$. In the lifted space, $\mathtt{HUR}$ reweights each projected coordinate by a scalar in $[\beta_n,1]$ and scales the residual by $\beta_n$, thereby attenuating selected components without reversing projected coordinates or increasing their magnitude.

\subsection{Controlled Prompt Update}
\label{sec:controlled_update}

Combining the magnitude control of $\mathtt{AKC}$ with the direction refinement of $\mathtt{HUR}$, $\mathtt{TPD}$ converts the base prompt-update into a controlled prompt-update:
\begin{equation}
\mathbf u_t
\;\longrightarrow\;
\Delta\mathbf p_t=g_t\mathbf u_t^{\mathrm{hur}},
\label{eq:main_final_update}
\end{equation}
where $\Delta \mathbf p_t$ denotes the controlled prompt update applied to the prompts at the current test step $t$.
Finally, Lemma~\ref{lem:akc} and the routing form in Eq.~\eqref{eq:main_hur} give the following controlled-update properties.

\begin{theorem}[Controlled prompt-update properties]
\label{thm:main}
Whenever the local Hankel window is available and $\mathtt{TPD}$ is applied, the controlled update $\Delta\mathbf p_t=g_t\mathbf u_t^{\mathrm{hur}}$ in Eq.~\eqref{eq:main_final_update} satisfies:
\begin{enumerate}[label=(\roman*),leftmargin=2em,itemsep=2pt,topsep=2pt]
\item \textbf{Hankel-lifted space non-amplification:} $\|\widetilde{\mathbf z}^{\mathrm{hur}}_t\|\le\|\mathbf z_t\|$, therefore $\|\Delta\mathbf p_t\|\le g_t\|\mathbf z_t\|\le\|\mathbf z_t\|$.
\item \textbf{Mode-wise routing without reversal:} each projected Hankel coordinate is reweighted as $w_{i,t}y_{i,t}$ for some $w_{i,t}\in[\beta_n,1]$, so no projected coordinate is reversed or amplified.
\item \textbf{Persistent preservation in $\mathtt{HUR}$ routing:} within the defined Hankel decomposition, the persistent component $\widetilde{\mathbf S}_t$ and its prompt-space extraction $\mathbf S_t$ are transmitted with coefficient $1$, while the oscillatory and residual components are attenuated by factors $\beta_c$ and $\beta_n$.
\end{enumerate}
\noindent\textit{\textbf{Proof sketch.}}
The proof follows from the weights $w_{i,t}\in[\beta_n,1]$ in the Hankel subspace, the decomposition $\mathbf z_t=Q_t\mathbf y_t+\widetilde{\mathbf e}_t$, and the bound $\|E_K\|_2=1$ (see Appendix~\ref{app:proof_theorem} for the full proof).
\end{theorem}

Since both modules act directly on the prompt update $\mathbf u_t$, $\mathtt{TPD}$ can be attached to different TPT methods while preserving their original update rule. Together, $\mathtt{AKC}$ and $\mathtt{HUR}$ regulate the prompt-update trajectory by combining magnitude control with component-wise direction routing.

\section{Experiments}
\subsection{Experimental Setup}
\label{sec:exp_setup}

\noindent\textbf{Datasets.}
Following previous TPT protocols~\citep{shu2022test}, we evaluate $\mathtt{TPD}$ on image classification under two distribution-shift settings, using ImageNet-pretrained backbones as source models and adapting them at test time on target datasets. The domain generalization benchmark includes ImageNet~\citep{deng2009imagenet} and four shifted variants: ImageNet-V2~\citep{recht2019imagenet}, ImageNet-Sketch~\citep{wang2019learning}, ImageNet-A~\citep{hendrycks2021natural}, and ImageNet-R~\citep{hendrycks2021many}. The cross-dataset benchmark contains 10 downstream datasets covering fine-grained recognition, textures, scenes, satellite images, and actions (see Appendix~\ref{app:exp_details} for the full details).

\noindent\textbf{Baselines.}
We attach $\mathtt{TPD}$ to TPT methods on two backbone families. For ViT-based visual prompt tuning, we evaluate four methods: VPT \citep{jia2022visual}, E2VPT \citep{han20232vpt}, VFPT \citep{zeng2024visual}, BPT \citep{wang2025attention}. For CLIP-based text prompt tuning, we evaluate both episodic and online methods: TPT \citep{shu2022test}, C-TPT \citep{yoon2024c}, R-TPT \citep{sheng2025r}, O-TPT \citep{sharifdeen2025tpt}, CAT \citep{zhang2025cat}, DyPt \citep{xiao2025dynaprompt}. We report both the original result and its gain with $\mathtt{TPD}$, denoted as ``Baseline~\parenclrbf{+$\mathtt{TPD}$}''.

\begin{table*}[t]
\caption{Comparisons on the domain generalization setting with visual and text TPT methods. OOD Mean averages the 4 variants excluding ImageNet, while Mean averages ImageNet and the 4 variants. Values are the baseline accuracy with $\mathtt{TPD}$ gains.}
\label{tab:dg}
\begin{center}
\begingroup
\renewcommand{\parenclrbf}[1]{{\bfseries\textcolor{MyText}{#1}}}
\fontsize{6pt}{9.5pt}\selectfont
\setlength{\tabcolsep}{5.0pt}
\renewcommand{\arraystretch}{0.85}
\resizebox{0.9\textwidth}{!}{%
\begin{tabular}{l ccccccc}
\toprule
Method & ImageNet & ImageNet-V2 & ImageNet-S & ImageNet-A & ImageNet-R & \textit{OOD Mean} & \textit{Mean} \\
\midrule
\rowcolor{MyBlue}
\multicolumn{8}{c}{\textit{ViT-based visual prompt tuning}} \\
\rowcolor{gray!12}
ViT-B/16
& 81.44 & 70.85 & 35.44 & 30.92 & 47.48 & 46.17 & 53.23 \\
VPT \parenclrbf{+$\mathtt{TPD}$}
& 78.16 \parenclrbf{+3.15} & 64.55 \parenclrbf{+5.32} & 32.77 \parenclrbf{+3.25} & 19.17 \parenclrbf{+19.81} & 42.14 \parenclrbf{+4.33} & 39.66 \parenclrbf{+8.18} & 47.36 \parenclrbf{+7.17} \\
E2VPT \parenclrbf{+$\mathtt{TPD}$}
& 78.35 \parenclrbf{+2.69} & 64.09 \parenclrbf{+4.93} & 32.89 \parenclrbf{+2.70} & 18.17 \parenclrbf{+20.79} & 42.07 \parenclrbf{+3.94} & 39.31 \parenclrbf{+8.09} & 47.11 \parenclrbf{+7.01} \\
VFPT \parenclrbf{+$\mathtt{TPD}$}
& 78.71 \parenclrbf{+2.62} & 65.34 \parenclrbf{+4.75} & 33.16 \parenclrbf{+2.81} & 22.05 \parenclrbf{+18.68} & 42.81 \parenclrbf{+3.55} & 40.84 \parenclrbf{+7.45} & 48.41 \parenclrbf{+6.48} \\
BPT \parenclrbf{+$\mathtt{TPD}$}
& 81.30 \parenclrbf{+2.87} & 70.48 \parenclrbf{+2.45} & 28.62 \parenclrbf{+9.93} & 38.97 \parenclrbf{+4.23} & 43.68 \parenclrbf{+5.39} & 45.44 \parenclrbf{+5.51} & 52.61 \parenclrbf{+4.97} \\
\midrule
\rowcolor{MyBlue}
\multicolumn{8}{c}{\textit{CLIP-based text prompt tuning}} \\
\rowcolor{gray!12}
CLIP 
& 66.73 & 60.86 & 46.09 & 47.87 & 73.98 & 57.20 & 59.11 \\
TPT \parenclrbf{+$\mathtt{TPD}$}
& 68.98 \parenclrbf{+2.21} & 63.45 \parenclrbf{+2.53} & 47.94 \parenclrbf{+1.68} & 54.77 \parenclrbf{+1.96} & 77.06 \parenclrbf{+2.92} & 60.81 \parenclrbf{+2.27} & 62.44 \parenclrbf{+2.26} \\
C-TPT \parenclrbf{+$\mathtt{TPD}$}
& 69.28 \parenclrbf{+2.35} & 63.36 \parenclrbf{+1.83} & 48.55 \parenclrbf{+2.02} & 52.93 \parenclrbf{+1.67} & 78.97 \parenclrbf{+3.85} & 60.95 \parenclrbf{+2.35} & 62.62 \parenclrbf{+2.34} \\
R-TPT \parenclrbf{+$\mathtt{TPD}$}
& 60.92 \parenclrbf{+5.08} & 64.91 \parenclrbf{+2.61} & 45.03 \parenclrbf{+2.36} & 48.92 \parenclrbf{+2.48} & 77.63 \parenclrbf{+3.01} & 59.12 \parenclrbf{+2.62} & 59.48 \parenclrbf{+3.11} \\
O-TPT \parenclrbf{+$\mathtt{TPD}$}
& 67.33 \parenclrbf{+4.52} & 61.65 \parenclrbf{+1.96} & 47.12 \parenclrbf{+1.81} & 49.87 \parenclrbf{+1.74} & 72.55 \parenclrbf{+4.59} & 57.80 \parenclrbf{+2.53} & 59.70 \parenclrbf{+2.92} \\
CAT \parenclrbf{+$\mathtt{TPD}$}
& 69.32 \parenclrbf{+2.50} & 60.91 \parenclrbf{+2.34} & 52.08 \parenclrbf{+2.23} & 58.69 \parenclrbf{+2.51} & 76.76 \parenclrbf{+3.32} & 62.11 \parenclrbf{+2.60} & 63.55 \parenclrbf{+2.58} \\
DyPt \parenclrbf{+$\mathtt{TPD}$}
& 69.61 \parenclrbf{+2.55} & 64.63 \parenclrbf{+2.47} & 48.22 \parenclrbf{+2.18} & 56.16 \parenclrbf{+2.39} & 78.15 \parenclrbf{+3.73} & 61.79 \parenclrbf{+2.69} & 63.35 \parenclrbf{+2.66} \\
\bottomrule
\end{tabular}}
\endgroup
\end{center}
\vspace{-5pt}
\end{table*}

\noindent\textbf{Implementation Details.}
We use $K=3$ updates per Hankel embedding and $M=8$ embeddings, giving $W=10$ past updates and seven transition pairs. These short windows balance temporal context with local estimation. The rank $r$ is selected adaptively from the current singular spectrum. We set $\bar\rho=1$, $g_{\min}=0.1$, and $(\beta_c,\beta_n)=(0.25,0.1)$ across all datasets, methods, and backbone families. The spectral threshold corresponds to the unit-circle boundary of the local modes; the gain floor avoids update shutdown, and the routing weights retain nonzero contributions from oscillatory and residual content. All experiments use an NVIDIA H100 GPU (see Appendix~\ref{app:hyper_sensitivity} for per-dataset hyperparameter experimental results).

\begin{table*}[t]
\caption{Comparisons on the cross-dataset setting with visual and text TPT methods. All methods are evaluated on 10 downstream datasets without training. Values are baseline accuracy with $\mathtt{TPD}$ gains.}
\label{tab:cds}
\begin{center}
\begingroup
\setlength{\tabcolsep}{1.8pt}
\renewcommand{\arraystretch}{1.45}
\resizebox{\textwidth}{!}{%
\begin{tabular}{l ccccccccccc}
\toprule
Method & Caltech101 & OxfordPets & StanfordCars & Flowers102 & Food101 & Aircraft & SUN397 & DTD & EuroSAT & UCF101 & \textit{Mean} \\
\midrule
\rowcolor{MyBlue}
\multicolumn{12}{c}{\textit{ViT-based visual prompt tuning}} \\
\rowcolor{gray!12}
ViT-B/16 
& 86.71 & 92.20 & 57.84 & 66.23 & 82.78 & 22.06 & 56.07 & 48.19 & 49.91 & 65.08 & 62.71 \\
VPT \parenclrbf{+$\mathtt{TPD}$}
& 92.65 \parenclrbf{+2.72} & 86.89 \parenclrbf{+2.49} & 60.68 \parenclrbf{+2.13} & 64.86 \parenclrbf{+3.23} & 85.14 \parenclrbf{+1.31} & 21.21 \parenclrbf{+3.25} & 59.92 \parenclrbf{+3.87} & 38.19 \parenclrbf{+4.06} & 20.91 \parenclrbf{+11.26} & 59.28 \parenclrbf{+2.73} & 58.97 \parenclrbf{+3.71} \\
E2VPT \parenclrbf{+$\mathtt{TPD}$}
& 92.29 \parenclrbf{+2.01} & 86.89 \parenclrbf{+2.75} & 60.03 \parenclrbf{+2.08} & 63.85 \parenclrbf{+3.96} & 84.55 \parenclrbf{+1.26} & 20.79 \parenclrbf{+3.51} & 59.07 \parenclrbf{+3.48} & 37.93 \parenclrbf{+5.33} & 17.29 \parenclrbf{+11.67} & 58.75 \parenclrbf{+2.81} & 58.14 \parenclrbf{+3.89} \\
VFPT \parenclrbf{+$\mathtt{TPD}$}
& 92.76 \parenclrbf{+2.36} & 87.33 \parenclrbf{+2.71} & 62.01 \parenclrbf{+2.41} & 67.77 \parenclrbf{+1.98} & 85.18 \parenclrbf{+1.43} & 21.69 \parenclrbf{+3.33} & 59.31 \parenclrbf{+2.64} & 40.53 \parenclrbf{+3.43} & 23.54 \parenclrbf{+12.66} & 60.26 \parenclrbf{+1.43} & 60.04 \parenclrbf{+3.44} \\
BPT \parenclrbf{+$\mathtt{TPD}$}
& 79.51 \parenclrbf{+13.89} & 68.26 \parenclrbf{+14.45} & 49.88 \parenclrbf{+11.15} & 64.82 \parenclrbf{+2.33} & 67.45 \parenclrbf{+15.21} & 21.47 \parenclrbf{+4.15} & 48.50 \parenclrbf{+7.84} & 41.46 \parenclrbf{+1.96} & 57.42 \parenclrbf{+1.71} & 50.52 \parenclrbf{+9.59} & 54.93 \parenclrbf{+8.23} \\
\midrule
\rowcolor{MyBlue}
\multicolumn{12}{c}{\textit{CLIP-based text prompt tuning}} \\
\rowcolor{gray!12}
CLIP 
& 93.35 & 88.25 & 65.48 & 67.44 & 83.65 & 23.67 & 62.59 & 44.27 & 42.01 & 65.13 & 63.58 \\
TPT \parenclrbf{+$\mathtt{TPD}$}
& 94.16 \parenclrbf{+2.16} & 87.79 \parenclrbf{+2.52} & 66.87 \parenclrbf{+1.68} & 68.98 \parenclrbf{+2.08} & 84.67 \parenclrbf{+3.65} & 24.78 \parenclrbf{+2.41} & 65.50 \parenclrbf{+3.01} & 47.75 \parenclrbf{+4.23} & 42.44 \parenclrbf{+3.59} & 68.04 \parenclrbf{+3.96} & 65.10 \parenclrbf{+2.93} \\
C-TPT \parenclrbf{+$\mathtt{TPD}$}
& 93.62 \parenclrbf{+1.93} & 88.23 \parenclrbf{+2.27} & 65.82 \parenclrbf{+1.31} & 69.84 \parenclrbf{+2.51} & 83.73 \parenclrbf{+4.62} & 24.02 \parenclrbf{+2.39} & 64.83 \parenclrbf{+3.28} & 46.04 \parenclrbf{+2.49} & 43.21 \parenclrbf{+2.87} & 65.74 \parenclrbf{+1.48} & 64.51 \parenclrbf{+2.52} \\
R-TPT \parenclrbf{+$\mathtt{TPD}$}
& 93.67 \parenclrbf{+2.09} & 87.24 \parenclrbf{+2.43} & 67.03 \parenclrbf{+1.34} & 68.71 \parenclrbf{+2.13} & 83.94 \parenclrbf{+4.02} & 23.92 \parenclrbf{+2.27} & 64.37 \parenclrbf{+3.61} & 46.43 \parenclrbf{+3.26} & 34.72 \parenclrbf{+3.08} & 67.23 \parenclrbf{+2.66} & 63.73 \parenclrbf{+2.69} \\
O-TPT \parenclrbf{+$\mathtt{TPD}$}
& 93.95 \parenclrbf{+2.22} & 87.95 \parenclrbf{+2.38} & 64.53 \parenclrbf{+1.24} & 70.07 \parenclrbf{+2.24} & 84.13 \parenclrbf{+4.18} & 23.64 \parenclrbf{+2.17} & 64.23 \parenclrbf{+3.49} & 45.68 \parenclrbf{+3.30} & 42.84 \parenclrbf{+3.15} & 64.16 \parenclrbf{+2.58} & 64.12 \parenclrbf{+2.70} \\
CAT \parenclrbf{+$\mathtt{TPD}$}
& 94.72 \parenclrbf{+2.36} & 90.34 \parenclrbf{+2.65} & 69.41 \parenclrbf{+1.52} & 72.73 \parenclrbf{+2.43} & 86.24 \parenclrbf{+4.41} & 24.13 \parenclrbf{+2.52} & 67.02 \parenclrbf{+3.88} & 47.24 \parenclrbf{+3.59} & 50.53 \parenclrbf{+3.46} & 69.52 \parenclrbf{+2.96} & 67.19 \parenclrbf{+2.98} \\
DyPt \parenclrbf{+$\mathtt{TPD}$}
& 94.32 \parenclrbf{+2.28} & 88.28 \parenclrbf{+2.61} & 67.65 \parenclrbf{+1.43} & 69.95 \parenclrbf{+2.18} & 85.42 \parenclrbf{+4.08} & 24.33 \parenclrbf{+2.36} & 66.32 \parenclrbf{+3.72} & 47.96 \parenclrbf{+3.64} & 42.28 \parenclrbf{+3.22} & 68.72 \parenclrbf{+3.08} & 65.52 \parenclrbf{+2.86} \\
\bottomrule
\end{tabular}}
\endgroup
\end{center}
\end{table*}

\subsection{Results}
For ViT-based visual TPT methods, $\mathtt{TPD}$ consistently improves all baselines under both domain generalization and cross-dataset transfer, as shown in Table~\ref{tab:dg} and Table~\ref{tab:cds}. The shifted-domain OOD mean gains range from $5.51$-$8.18\%$, while the cross-dataset mean gains range from $3.44$-$8.23\%$. The improvements from adding $\mathtt{TPD}$ are especially large under hard shifts; for example, VPT gains $19.81\%$ on ImageNet-A, and BPT gains $8.23\%$ on average across cross-dataset transfer. This suggests that training-time visual prompts remain vulnerable when reused for naive test-time entropy updates, where update magnitudes may expand and directions may oscillate or drift. These larger gains therefore reflect the lack of dedicated TTA stabilizers for ViT prompt tuning: $\mathtt{TPD}$ supplies the missing trajectory control and turns otherwise unstable test-time tuning into effective adaptation.

For CLIP-based text TPT methods, $\mathtt{TPD}$ also improves all episodic and online baselines in Table~\ref{tab:dg} and Table~\ref{tab:cds}. On domain generalization, the OOD mean gains are consistently between $2.27\%$ and $2.69\%$; on cross-dataset transfer, the mean gains range from $2.52\%$ to $2.98\%$. For example, DyPt improves from $61.79\%$ to $64.48\%$ in the domain generalization OOD mean. 
Although these gains are more moderate than those for visual prompts, they remain highly consistent across different adaptation designs. By acting directly on the base update vectors, $\mathtt{TPD}$ stabilizes text-prompt trajectories without replacing the original text-prompt objective. This consistency further shows that $\mathtt{TPD}$ serves as a prompt-agnostic module complementary to text TPT methods.

\begin{wraptable}{r}{0.48\textwidth}
\vspace{-12pt}
\caption{Comparison with simple stabilization methods on visual and text TPT methods. DG and CSD denote the accuracy means over four ImageNet variants and 10 cross-datasets. $\times$: per-step wall-clock overhead relative to the baseline.}
\label{tab:simple_stabilization}
\begin{center}
\scriptsize
\setlength{\tabcolsep}{2.5pt}
\begin{tabular}{l ccc ccc}
\toprule
& \multicolumn{3}{c}{VPT} & \multicolumn{3}{c}{TPT} \\
\cmidrule(lr){2-4} \cmidrule(lr){5-7}
Variant & \textit{DG} & \textit{CSD} & $\times$ & \textit{DG} & \textit{CSD} & $\times$ \\
\midrule
Baseline    & 39.66 & 58.97 & 1.00 & 60.81 & 65.10 & 1.00 \\
\rowcolor{MyBlue}
\multicolumn{7}{c}{\textit{Simple stabilization rules}} \\
+ Update clipping   & 41.92 & 59.54 & 1.11 & 61.36 & 65.47 & 1.06 \\
+ EMA update   & 42.74 & 59.86 & 0.99 & 61.65 & 65.82 & 1.01 \\
+ Norm control   & 43.65 & 61.03 & 1.07 & 61.72 & 65.74 & 1.03 \\
+ Low-pass filtering  & 44.28 & 60.91 & 1.00 & 62.23 & 66.23 & 1.00 \\
\rowcolor{MyBlue}
\multicolumn{7}{c}{\textit{Prompt-update decomposition and stabilization}} \\
+$\mathtt{TPD}$  & 47.84 & 62.68 & 1.13 & 63.08 & 68.03 & 1.08 \\
\bottomrule
\end{tabular}
\end{center}
\end{wraptable}

We further compare $\mathtt{TPD}$ with simple stabilization rules in Table \ref{tab:simple_stabilization}, including update clipping~\citep{pascanu2013difficulty}, EMA updates~\citep{tarvainen2017mean}, norm control~\citep{niu2023towards}, and low-pass filtering~\citep{zhang2019lookahead}.
These heuristics improve the TPT baselines but remain below $\mathtt{TPD}$: on VPT, low-pass filtering improves DG from $39.66$ to $44.28$, while $\mathtt{TPD}$ reaches $47.84$; on TPT, it improves CSD from $65.10$ to $66.23$, whereas $\mathtt{TPD}$ reaches $68.03$. 
$\mathtt{TPD}$ also incurs modest overhead, with $1.13\times$ on VPT and $1.08\times$ on TPT. These results suggest that simple scale control or short-term smoothing is insufficient, while $\mathtt{TPD}$ further uses component-wise routing to stabilize prompt-update dynamics across both visual and text prompts under TTA (see Appendix~\ref{app:full_per_dataset} for the per-dataset simple-stabilization results).

\begin{wraptable}{r}{0.45\textwidth}
\vspace{-13pt}
\caption{Ablation study of $\mathtt{TPD}$ on visual and text TPT methods. DG and CSD denote the accuracy means over four ImageNet variants and 10 cross-datasets.}
\label{tab:ablation}
\begin{center}
\scriptsize
\begin{tabular}{l cc cc}
\toprule
& \multicolumn{2}{c}{VPT} & \multicolumn{2}{c}{TPT} \\
\cmidrule(lr){2-3} \cmidrule(lr){4-5}
Variant & \textit{DG} & \textit{CSD} & \textit{DG} & \textit{CSD} \\
\midrule
Baseline  & 39.66 & 58.97 & 60.81 & 65.10 \\
\rowcolor{MyBlue}
\multicolumn{5}{c}{\textit{Ablation study of Koopman operator modes}} \\
Full    & 43.58 & 60.86 & 61.80 & 66.38 \\
Diagonal (+$\mathtt{AKC}$)  & 45.92 & 61.74 & 62.56 & 67.21 \\
\rowcolor{MyBlue}
\multicolumn{5}{c}{\textit{Ablation study of projection matrix modes}} \\
Plain & 42.31 & 60.22 & 61.43 & 66.01 \\
Hankel (+$\mathtt{HUR}$)   & 45.98 & 62.18 & 62.84 & 67.54 \\
\rowcolor{MyBlue}
\multicolumn{5}{c}{\textit{Ablation study of $\mathtt{HUR}$ routing}} \\
w/o Pres. & 47.06 & 62.29 & 62.95 & 67.73 \\
w/o Osc.  & 47.28 & 62.37 & 62.84 & 67.87   \\
w/o Res.  & 46.84 & 62.21 & 62.80 & 67.54   \\
\rowcolor{MyBlue}
\multicolumn{5}{c}{\textit{Full}} \\
+$\mathtt{TPD}$ & 47.84 & 62.68 & 63.08 & 68.03 \\
\bottomrule
\end{tabular}
\end{center}
\vspace{-3pt}
\end{wraptable}

\subsection{Ablation Study}
We conduct ablations on ViT-based VPT and CLIP-based TPT to examine three design choices in $\mathtt{TPD}$: the Koopman operator for $\mathtt{AKC}$, the projection matrix for $\mathtt{HUR}$, and the routing weights in $\mathtt{HUR}$. The first block compares a full Koopman mode with the proposed diagonal mode in Table~\ref{tab:ablation}. 
The diagonal mode performs better, improving VPT-DG from $43.58$ to $45.92$ and TPT-CSD from $66.38$ to $67.21$. This shows the short Hankel window makes cross-step coupling difficult to estimate reliably, while the diagonal mode provides stable scalar coefficients for spectral control.

The second block compares plain projection with the Hankel embedding used by $\mathtt{HUR}$. Hankel projection consistently outperforms plain projection, improving VPT-DG from $42.31$ to $45.98$ and TPT-CSD from $66.01$ to $67.54$, indicating that delay embedding captures cross-step update dynamics missed by a static projection. 
Furthermore, the third block ablates $\mathtt{HUR}$ routing: w/o Pres., w/o Osc., and w/o Res. use $(\beta_c,\beta_c,\beta_n)$, $(1,1,\beta_n)$, and $(1,\beta_c,1)$, respectively, instead of the full $(1,\beta_c,\beta_n)$. All three variants underperform $\mathtt{TPD}$ across VPT and TPT, supporting the routing strategy that preserves persistent directions, softly attenuates oscillatory directions, and strongly suppresses unexplained perturbations (see Appendix~\ref{app:abla_per_dataset} for the full ablation results).

\subsection{Analysis}

To quantify the stability of visual prompt-update trajectories, we visualize the Koopman spectra and decomposition dynamics in Figure~\ref{fig3}. VPT concentrates most spectral mass near the origin, indicating that many update modes decay rapidly rather than carrying persistent adaptation. This does not contradict the large $\rho$ values in Figure~\ref{fig3}(c), since the spectrum aggregates eigenvalue density over the trajectory, while $\rho$ tracks the largest local eigenvalue at each step. Thus, a trajectory can be dominated by decaying modes overall but still contain intermittent expansive steps. Consistently, VPT shows residual-dominated directions and frequent excursions of $\rho$ beyond the stable reference, suggesting that its updates are both noisy in direction and unstable in magnitude. After applying $\mathtt{TPD}$, the spectrum shifts toward the unit circle and becomes more evenly distributed, indicating that useful modes are preserved. Meanwhile, the persistent ratio increases, oscillatory and residual ratios decrease, and $\rho$ stays closer to the stable regime. These results show that $\mathtt{HUR}$ improves the directional composition of updates, while $\mathtt{AKC}$ controls their magnitude expansion.

\begin{figure}[t]
    \centering
    \includegraphics[width=\linewidth]{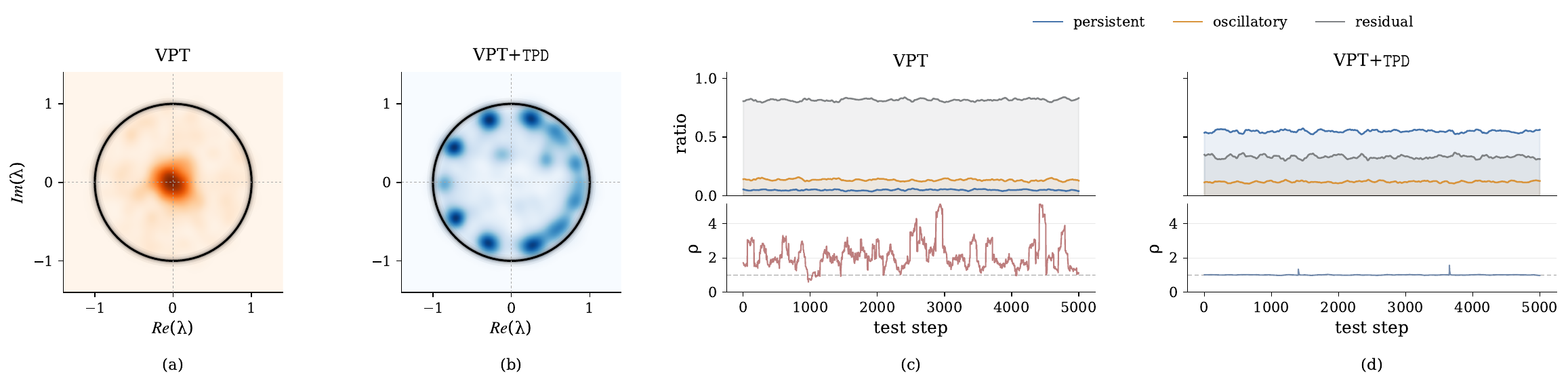}
    \caption{(a), (b) show Koopman spectra of visual prompt trajectories. The spectra are estimated post-hoc from generated update trajectories, and the black circle denotes the unit circle. (c), (d) show component ratios and spectral radius $\rho$ over test steps.}
    \label{fig3}
\end{figure}

\begin{figure}[t]
    \centering
    \includegraphics[width=\linewidth]{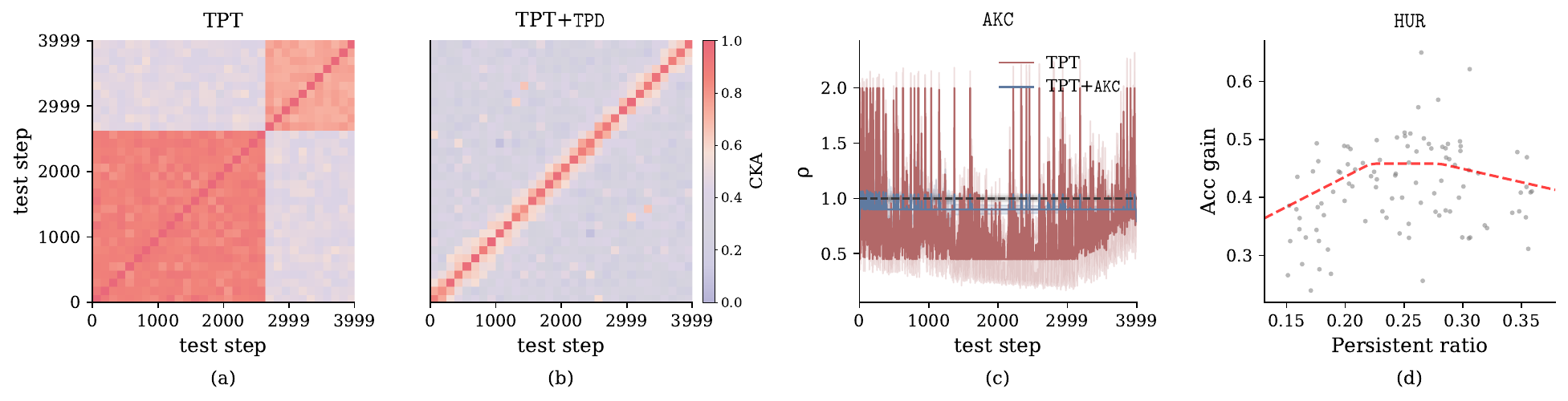}
    \caption{(a), (b) CKA similarity between text-prompt representations across test steps after smoothing. (c) Koopman spectral radius estimated from prompt updates, the estimate is post-hoc for TPT and used online by $\mathtt{AKC}$. (d) Persistent ratio of $\mathtt{HUR}$ versus subsequent accuracy gain. Each gray point denotes one short temporal window. The red dashed curve is a piecewise linear fitted trend line.}
    \label{fig4}
\end{figure}

We further analyze text TPT methods after adding $\mathtt{TPD}$. The CKA~\citep{kornblith2019similarity} matrices measure pairwise similarity between text-prompt representations across test steps. In Figure~\ref{fig4}(a), TPT forms coarse high-similarity blocks, indicating that prompts collapse into a few similar states within local segments and change abruptly across segments. In contrast, TPT+$\mathtt{TPD}$ preserves a sharper near-diagonal structure in Figure~\ref{fig4}(b), suggesting smoother stepwise evolution without excessive cross-step collapse. Figure~\ref{fig4}(c) further shows that the Koopman spectral radius of TPT fluctuates strongly and frequently exceeds $\rho=1$, whereas $\mathtt{AKC}$ keeps the radius closer to the stable reference. Finally, Figure~\ref{fig4}(d) links the direction signal to performance: accuracy gain generally increases from low to moderate persistent ratios and then reaches a peak. This trend supports the use of persistent ratio as a routing signal for $\mathtt{HUR}$, since low persistence corresponds to oscillatory or residual-dominated updates, while moderate persistence reflects useful adaptation structure (see Appendix~\ref{app:persistent_gain} for the visualizations on all datasets).

\begin{figure}[t]
    \centering
    \includegraphics[width=\linewidth]{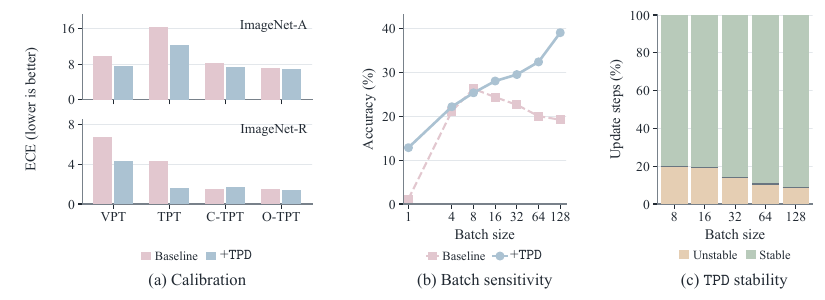}
    \caption{(a) ECE on ImageNet-A and ImageNet-R, with separate vertical scales. (b) Accuracy across batch sizes. Pink and blue in (a,b) denote each baseline and its $\mathtt{TPD}$ variant, respectively. (c) Proportions of stable ($\rho_t\leq1$, light green) and unstable ($\rho_t>1$, light orange) $\mathtt{TPD}$ update steps, where $\rho_t$ denotes the estimated spectral radius; each stacked bar sums to $100\%$.}
    \label{fig5}
\end{figure}

\subsection{Robustness and Reliability}

We evaluate calibration across four prompt-tuning methods on ImageNet-A and ImageNet-R (Fig.~\ref{fig5}(a)). $\mathtt{TPD}$ reduces expected calibration error (ECE) in seven of eight comparisons, with C-TPT on ImageNet-R as the only exception. For TPT, ECE decreases from $16.42\%$ to $12.27\%$ on ImageNet-A and from $4.36\%$ to $1.62\%$ on ImageNet-R. These improvements suggest that controlling prompt dynamics also regularizes predictive confidence, extending the benefits of stabilization to the alignment between confidence and empirical accuracy across multiple prompt-tuning formulations.

We examine VPT batch-size sensitivity on ImageNet-A with top-$10\%$ confidence filtering (Fig.~\ref{fig5}(b)). $\mathtt{TPD}$ improves accuracy in six of seven settings, with a decrease at $B=8$. At $B=1$, accuracy increases from $1.14\%$ to $12.82\%$, and at $B=128$, from $19.17\%$ to $38.98\%$. $\mathtt{TPD}$ improves over the baseline in most small-batch settings, although absolute accuracy remains low at very small batch sizes. Moreover, baseline accuracy declines beyond $B=8$, while $\mathtt{TPD}$ accuracy continues to improve with increasing batch size. This divergent trend highlights the importance of update regulation in converting a larger sample budget into more effective adaptation.

We further assess the proportions of spectrally stable and unstable $\mathtt{TPD}$ updates (Fig.~\ref{fig5}(c)). The stable fraction increases from $80.07\%$ at $B=8$ to $91.16\%$ at $B=128$, with a consistent upward trend across the five measured batch sizes. This pattern suggests that larger batches provide a stronger basis for estimating and regulating local dynamics, reducing the frequency of unstable update behavior. Viewed alongside the accuracy trends, it supports a link between the reliability of local dynamics and adaptation quality. The stable fraction complements accuracy by characterizing the update process itself, allowing predictive performance and dynamical reliability to be assessed separately (see Appendices~\ref{app:rb_selected_streams}--\ref{app:rb_selected_controllers}
for results on nonstationary streams, component ratios,
seed variation, and adaptive magnitude controllers). 

\section{Conclusion}
We propose $\mathtt{TPD}$, which characterizes observed prompt-update trajectories from a dynamical-systems view and decomposes base prompt updates into magnitude and direction.
Specifically, $\mathtt{TPD}$ projects recent updates into a Hankel subspace, $\mathtt{AKC}$ constrains unstable update magnitudes, and $\mathtt{HUR}$ preserves persistent components while attenuating oscillatory and residual components.
As a prompt-agnostic framework, $\mathtt{TPD}$ operates directly on base prompt updates and serves as a plug-in stabilizer without modifying the backbone, prompt architecture, or adaptation loss.
Across visual and text prompts, $\mathtt{TPD}$ consistently improves TPT baselines under domain generalization and cross-dataset settings.

\textbf{Limitation.} $\mathtt{TPD}$ relies on a meaningful recent update history and may be less effective when updates are sparse, noisy, or rapidly changing. Its finite-window local model does not guarantee global convergence or semantic correctness of retained directions, and soft routing may attenuate useful corrections. Small-batch adaptation also remains challenging.

\newpage
\bibliographystyle{plainnat}
\bibliography{ref}

%%%%%%%%%%%%%%%%%%%%%%%%%%%%%%%%%%%%%%%%%%%%%%%%%%%%%%%%%%%%
\newpage
\appendix

\section{Spectral Justification of the Three-Component Decomposition}
\label{app:spectral_decomposition}

This section justifies the three-component decomposition in Eq.~\eqref{eq:main_components} from the Hankel-DMD spectral view. The key point is that $\mathtt{AKC}$ and $\mathtt{HUR}$ use different aspects of the same local Koopman spectrum. The spectral radius $\rho(A_t)$ depends on eigenvalue magnitudes and diagnoses local magnitude expansion for $\mathtt{AKC}$. In contrast, the persistent and oscillatory split depends on spectral phase, or on the sign of each scalar coefficient under the diagonal local model, and is used by $\mathtt{HUR}$ to route update directions.

\begin{lemma}[Hankel-DMD spectral reconstruction]
\label{lem:hankel_dmd_spectrum}
Consider the projected Hankel trajectory in the local subspace at test step $t$, and suppose that within the short window it is represented by a finite-dimensional local Koopman model
\begin{equation}
\mathbf y_{\tau+1}\approx A_t\mathbf y_{\tau}.
\label{eq:app_local_linear}
\end{equation}
If $A_t$ is diagonalizable with eigenpairs $\{(\lambda_j,\boldsymbol\phi_j)\}_{j=1}^r$, then the local DMD reconstruction can be written as
\begin{equation}
\mathbf y_{\tau}
\approx
\sum_{j=1}^{r} b_j\boldsymbol\phi_j\lambda_j^{\tau},
\qquad
\lambda_j=|\lambda_j|e^{\mathrm i\theta_j},
\qquad
\omega_j=\log(\lambda_j)=\alpha_j+\mathrm i\beta_j.
\label{eq:app_dmd_reconstruction}
\end{equation}
Here $|\lambda_j|$ (or $\alpha_j$) controls growth or decay, while $\theta_j$ (or $\beta_j$) controls phase rotation and hence oscillatory behavior.
\end{lemma}

\paragraph{Proof.}
The DMD reconstruction follows from expanding the initial projected state in the eigenbasis of $A_t$: if $\mathbf y_0=\sum_j b_j\boldsymbol\phi_j$, then $A_t^\tau\mathbf y_0=\sum_j b_j\boldsymbol\phi_j\lambda_j^\tau$. Writing $\lambda_j=|\lambda_j|e^{\mathrm i\theta_j}$ gives $\lambda_j^\tau=|\lambda_j|^\tau e^{\mathrm i\theta_j\tau}$, which separates magnitude evolution from phase evolution. Equivalently, using the continuous-time parameter $\omega_j=\log(\lambda_j)=\alpha_j+\mathrm i\beta_j$, each mode evolves as $e^{\omega_j\tau}=e^{\alpha_j\tau}e^{\mathrm i\beta_j\tau}$. Thus eigenvalue magnitude controls whether the mode grows or decays, while eigenvalue phase controls whether the mode rotates or oscillates. \qed

\begin{lemma}[Spectral partition into persistent, oscillatory, and residual parts]
\label{lem:spectral_partition}
For a real-valued local trajectory reconstructed by Eq.~\eqref{eq:app_dmd_reconstruction}, the retained Hankel-DMD modes admit the following partition.
\begin{itemize}[leftmargin=2em,itemsep=2pt,topsep=2pt]
\item Modes with $\theta_j=0$ have positive real eigenvalues and produce non-oscillatory exponential evolution. These modes correspond to persistent content.
\item Modes with $\theta_j=\pi$ have negative real eigenvalues and produce alternating-sign dynamics. These modes correspond to oscillatory content.
\item Non-real conjugate pairs with $\theta_j\notin\{0,\pi\}$ combine into exponentially modulated sine and cosine terms, and therefore correspond to oscillatory content.
\item The part not explained by the retained Hankel-DMD reconstruction is residual content.
\end{itemize}
\end{lemma}

\paragraph{Proof.}
For a positive real eigenvalue, $\lambda_j=|\lambda_j|$ and the corresponding scalar time factor is $|\lambda_j|^\tau=e^{\alpha_j\tau}$, which has no phase rotation and does not alternate sign. It is therefore non-oscillatory within the local model. For a negative real eigenvalue, $\lambda_j=-|\lambda_j|=|\lambda_j|e^{\mathrm i\pi}$, so $\lambda_j^\tau=|\lambda_j|^\tau(-1)^\tau$, yielding alternating-sign behavior across steps. For a non-real pair $\lambda_j,\overline{\lambda_j}$, the real-valued contribution of the pair is
\begin{equation}
b_j\boldsymbol\phi_j e^{(\alpha_j+\mathrm i\beta_j)\tau}
+\overline{b_j\boldsymbol\phi_j}\,e^{(\alpha_j-\mathrm i\beta_j)\tau}
=
e^{\alpha_j\tau}
\bigl(\mathbf a_j\cos(\beta_j\tau)+\mathbf b_j\sin(\beta_j\tau)\bigr),
\label{eq:app_conjugate_pair}
\end{equation}
for real vectors $\mathbf a_j,\mathbf b_j$. This is an exponentially modulated sinusoid and is therefore oscillatory. Finally, any discrepancy between the current projected update and the retained local DMD reconstruction is, by definition, unexplained by the retained modes and forms residual content. \qed

\paragraph{Connection to the diagonal implementation.}
The full spectral partition above explains the origin of the three components. The implementation in the main text uses the diagonal local model $A_t=\operatorname{diag}(\mathbf a_t)$, so each projected coordinate is a scalar local DMD mode. Accordingly, the products, divisions, powers, and nonlinear functions in Eqs.~\eqref{eq:main_koopman}--\eqref{eq:main_components} are applied coordinate-wise. For coordinate $i$, a positive coefficient $a_{i,t}>0$ has phase $0$ and represents same-sign, non-oscillatory continuation; a negative coefficient $a_{i,t}<0$ has phase $\pi$ and represents sign-flipping oscillation. The magnitude $|a_{i,t}|$ only describes local amplification or decay and is used by $\rho(A_t)$ for $\mathtt{AKC}$; it does not determine whether the coordinate is persistent or oscillatory. This separates the magnitude branch from the direction branch.

The gates in Eq.~\eqref{eq:main_ratios} use these coordinate-wise quantities to assign the observed projected update in Eq.~\eqref{eq:main_components} to persistent, oscillatory, and residual parts. Specifically, $\boldsymbol\pi_t=(1+\tanh(\mathbf a_t))/2$ softly separates same-sign continuation from sign-flipping dynamics, while
\begin{equation}
q_{i,t}
=
\frac{\hat y_{i,t}^{\,2}}{\hat y_{i,t}^{\,2}+r_{i,t}^{\,2}}
=
\left(1+\frac{r_{i,t}^{\,2}}{\hat y_{i,t}^{\,2}}\right)^{-1},
\label{eq:app_q_snr}
\end{equation}
for $\hat y_{i,t}\neq0$, and the original ratio is well-defined for all coordinates. Thus $q_{i,t}$ is a monotone function of the prediction signal-to-residual ratio: it is close to one when the local DMD prediction explains the current coordinate and close to zero when the residual dominates. Therefore Eq.~\eqref{eq:main_components} is an observation-centric soft spectral decomposition: the DMD model determines the gates, while the routed content remains proportional to the observed update coordinate $y_{i,t}$. This avoids synthesizing a new update from the prediction and keeps $\mathtt{HUR}$ aligned with the current base prompt update.

Combining Eq.~\eqref{eq:main_components} with the lifted residual in Eq.~\eqref{eq:main_projection_current}, the prompt-space components satisfy $\mathbf u_t=\mathbf S_t+\mathbf C_t+\mathbf N_t$ as used in Eq.~\eqref{eq:main_lift}. Thus the residual branch contains both the projected unexplained component $(1-\mathbf q_t)\mathbf y_t$ and the residual $\widetilde{\mathbf e}_t$. The subsequent theorem uses this spectral partition only for bounded routing; $\mathtt{HUR}$ does not use $\rho(A_t)$ to decide persistent versus oscillatory content, and $\mathtt{AKC}$ does not use the phase partition to scale magnitudes.

\section{Proof of Lemma~\ref{lem:akc}}
\label{app:proof_lemma}

Since $g_t\ge 0$, scalar multiplication by $g_t$ scales every eigenvalue of $A_t$ by $g_t$. In the diagonal case used in Eq.~\eqref{eq:main_koopman}, this is immediate:
\begin{equation}
\rho(g_tA_t)
\;=\;\max_i |g_t a_{i,t}|
\;=\;g_t\max_i |a_{i,t}|
\;=\;g_t\rho(A_t).
\label{eq:app_diag}
\end{equation}
By Eq.~\eqref{eq:main_gain}, $g_t\in[g_{\min},1]$, hence $g_t\rho(A_t)\le\rho(A_t)$ with equality iff $g_t=1$. From the definition of $g_t$, $g_t=1$ if and only if $\rho_t\le\bar\rho$. For the upper bound Eq.~\eqref{eq:lem_akc_upper}, we split $\rho_t$ into three regimes according to Eq.~\eqref{eq:main_gain}: if $\rho_t\le\bar\rho$, then $g_t=1$, so $g_t\rho_t=\rho_t\le\bar\rho$; if $\bar\rho<\rho_t\le\bar\rho/g_{\min}$, then $\bar\rho/\rho_t\ge g_{\min}$, so $g_t=\bar\rho/\rho_t$ and $g_t\rho_t=\bar\rho$; if $\rho_t>\bar\rho/g_{\min}$, then $\bar\rho/\rho_t<g_{\min}$, so $g_t=g_{\min}$ and $g_t\rho_t=g_{\min}\rho_t$. Combining the three cases with Eq.~\eqref{eq:app_diag}, we obtain $\rho(g_tA_t)=g_t\rho_t\le\max(\bar\rho,g_{\min}\rho_t)$, which gives Eq.~\eqref{eq:lem_akc_upper}. In particular, when $\rho_t\le\bar\rho/g_{\min}$, the first two cases apply and $g_t\rho_t\le\bar\rho$. \qed

\paragraph{Remark (diagonal vs.\ full Koopman).}
The spectral homogeneity in Lemma~\ref{lem:akc} does not require $A_t$ to be diagonal. The diagonal local model is instead used to provide scalar mode coefficients $\{a_{i,t}\}_{i=1}^{r}$ for the gates in Eq.~\eqref{eq:main_ratios}, the coordinate-wise decomposition in Eq.~\eqref{eq:main_components}, and the mode-wise routing property in Theorem~\ref{thm:main}. The following lemma summarizes why this coordinate-wise estimator is preferable in the local Hankel regime. Let $A_t^{\mathrm{full}}=Y_t^+(Y_t^-)^\dagger$ be the full-matrix alternative. Under the fixed-design local linear model $\mathbf y_{\tau+1}=A^\star\mathbf y_\tau$, define $G\coloneqq Y_t^-(Y_t^-)^\top$. Since $Y_t^-$ contains $M-1$ transition columns and $r\le\min(Kd,M)$, the boundary case $r\ge M$ can arise when the selected rank reaches the number of Hankel columns.

\begin{lemma}[Diagonal estimator in the local Hankel regime]
\label{lem:diag_vs_full}
Under the setting above, the full Koopman estimator $A_t^{\mathrm{full}}$ and the diagonal estimator $A_t=\operatorname{diag}(a_{1,t},\dots,a_{r,t})$ satisfy:
\begin{itemize}[leftmargin=2em,itemsep=2pt,topsep=2pt]
\item[(i)] \emph{Underdetermined full estimator.} If $r\ge M$, then $\mathrm{rank}(Y_t^-)\le M-1<r$, so the least-squares minimizers for $AY_t^-\approx Y_t^+$ form an affine set with at least $r(r-M+1)$ free dimensions. The Moore--Penrose estimator selects one minimizer, but the dynamics on $\ker((Y_t^-)^\top)$ remain unconstrained by the local data. In contrast, each diagonal coefficient in Eq.~\eqref{eq:main_koopman} is a scalar least-squares estimate and is unique for every mode satisfying $\sum_{\tau=t-M}^{t-2}y_{i,\tau}^2>0$.
\item[(ii)] \emph{Local estimation variance.} If $r<M$ and $G\succ0$, the full least-squares estimator satisfies $\mathbb{E}\|\hat A^{\mathrm{full}}-A^\star\|_F^2=r\sigma^2\mathrm{tr}(G^{-1})$ under the fixed-design noise model. For the mode-wise scalar estimator in Eq.~\eqref{eq:main_koopman}, the conditional noise variance of $\hat a_{i,t}$ is $\sigma^2/G_{ii}$. Thus the full estimator estimates coupled cross-coordinate effects from a short local window, whereas the diagonal estimator estimates only the scalar coefficient used by the corresponding gate.
\item[(iii)] \emph{Compatibility with the decomposition.} The gates in Eq.~\eqref{eq:main_ratios} and the decomposition in Eq.~\eqref{eq:main_components} are defined for each Hankel coordinate through the scalar coefficient $a_{i,t}$. A full matrix mixes coordinates and therefore does not provide the same coordinate-wise quantities $\{a_{i,t}\}_{i=1}^r$ required to form $q_{i,t}$, $\pi_{i,t}$, and the weight $w_{i,t}\in[\beta_n,1]$ in Theorem~\ref{thm:main}(ii).
\item[(iv)] \emph{Per-step cost.} Computing $A_t$ costs $\mathcal{O}(rM)$ for the diagonal estimator versus $\mathcal{O}(r^2M+r^3)$ for the full estimator.
\end{itemize}
\end{lemma}

\section{Proof of Theorem~\ref{thm:main}}
\label{app:proof_theorem}

Let
\[
\widetilde{\mathbf e}_t=(I_{Kd}-Q_tQ_t^\top)\mathbf z_t
\]
be the Hankel-lifted residual from Eq.~\eqref{eq:main_projection_lifted}. The lifted routed update is
\begin{equation}
\widetilde{\mathbf z}^{\mathrm{hur}}_t
=
Q_t\bigl(\mathbf s_t^\star+\beta_c\mathbf c_t^\star+\beta_n\mathbf n_t^\star\bigr)
+\beta_n\widetilde{\mathbf e}_t,
\qquad
\mathbf u_t^{\mathrm{hur}}=E_K\widetilde{\mathbf z}^{\mathrm{hur}}_t .
\label{eq:app_lifted_hur}
\end{equation}
Combining Eq.~\eqref{eq:main_components} with Eq.~\eqref{eq:app_lifted_hur}, the $i$-th routed Hankel coordinate satisfies
\begin{equation}
s_{i,t}^{\star}+\beta_c c_{i,t}^{\star}+\beta_n n_{i,t}^{\star}
=
\bigl[
q_{i,t}\pi_{i,t}
+\beta_c q_{i,t}(1-\pi_{i,t})
+\beta_n(1-q_{i,t})
\bigr]y_{i,t}
=:w_{i,t}y_{i,t}.
\label{eq:app_w}
\end{equation}
Since $q_{i,t},\pi_{i,t}\in[0,1]$ and $0\le\beta_n\le\beta_c\le1$, and since $q_{i,t}\pi_{i,t}+q_{i,t}(1-\pi_{i,t})+(1-q_{i,t})=1$, the coefficient $w_{i,t}$ is a convex combination of $1$, $\beta_c$, and $\beta_n$. Therefore,
\begin{equation}
w_{i,t}\in[\beta_n,1].
\label{eq:app_w_bound}
\end{equation}

\paragraph{Proof of (i).}
Because $Q_t$ has orthonormal columns and $\widetilde{\mathbf e}_t\perp\mathrm{range}(Q_t)$ in the Hankel-lifted space,
\begin{equation}
\|\widetilde{\mathbf z}^{\mathrm{hur}}_t\|^2
=
\sum_{i=1}^{r} w_{i,t}^{\,2}y_{i,t}^{\,2}
+\beta_n^{\,2}\|\widetilde{\mathbf e}_t\|^2
\le
\sum_{i=1}^{r} y_{i,t}^{\,2}
+\|\widetilde{\mathbf e}_t\|^2
=
\|\mathbf z_t\|^2,
\label{eq:app_hankel_nonamp}
\end{equation}
where we used $w_{i,t}\le1$ and $\beta_n\le1$. Since $\|E_K\|_2=1$ and $g_t\in(0,1]$, we have
\begin{equation}
\|\Delta\mathbf p_t\|
=
g_t\|\mathbf u_t^{\mathrm{hur}}\|
=
g_t\|E_K\widetilde{\mathbf z}^{\mathrm{hur}}_t\|
\le
g_t\|\widetilde{\mathbf z}^{\mathrm{hur}}_t\|
\le
g_t\|\mathbf z_t\|
\le
\|\mathbf z_t\|.
\label{eq:app_prompt_bound}
\end{equation}
This proves the Hankel-space non-amplification statement and the induced prompt-update bound.

\paragraph{Proof of (ii).}
From Eq.~\eqref{eq:app_w}, the projected Hankel coordinate after routing equals $w_{i,t}y_{i,t}$. Since $w_{i,t}\ge\beta_n\ge0$, the sign of the projected coordinate is not reversed; since $w_{i,t}\le1$, its magnitude is not amplified.

\paragraph{Proof of (iii).}
The lifted persistent component $\widetilde{\mathbf S}_t$ is transmitted with coefficient $1$, while $\widetilde{\mathbf C}_t$ and $\widetilde{\mathbf N}_t$ are attenuated by factors $\beta_c$ and $\beta_n$, respectively. Applying the linear extraction $E_K$ gives the same coefficients for their prompt-space extractions $\mathbf S_t$, $\mathbf C_t$, and $\mathbf N_t$ in Eq.~\eqref{eq:main_hur}. \qed

\section{Full Experiments}
Together with the main tables, these appendix results form a multi-axis stress test targeted at the core claim of $\mathtt{TPD}$: prompt-agnostic stabilization of update trajectories rather than dataset- or baseline-specific tuning. The evaluation spans datasets, prompt modalities, backbone families, adaptation modes, distribution-shift protocols, TPT baselines, simple stabilization rules, and component ablations, making the improvements difficult to explain by a narrow choice of dataset, prompt type, backbone, or heuristic stabilizer. 

\subsection{Full Experimental Details}
\label{app:exp_details}

\paragraph{Datasets.}
Following previous methods~\citep{shu2022test}, we evaluate our method on 15 datasets under two settings that suffer from distribution shifts: domain generalization and cross-dataset transfer. For the domain generalization setting, we evaluate on ImageNet~\citep{deng2009imagenet} and its four OOD variants: ImageNet-V2~\citep{recht2019imagenet}, ImageNet-Sketch~\citep{wang2019learning}, ImageNet-A~\citep{hendrycks2021natural}, and ImageNet-R~\citep{hendrycks2021many}. For the cross-dataset setting, we evaluate on 10 datasets covering diverse tasks: Caltech101~\citep{fei2004learning}, OxfordPets~\citep{parkhi2012cats}, StanfordCars~\citep{krause20133d}, Flowers102~\citep{nilsback2008automated}, Food101~\citep{bossard2014food}, FGVCAircraft~\citep{maji2013fine}, SUN397~\citep{xiao2010sun}, DTD~\citep{cimpoi2014describing}, EuroSAT~\citep{helber2019eurosat}, and UCF101~\citep{soomro2012ucf101}.

\paragraph{Baseline Models.}
We attach $\mathtt{TPD}$ to state-of-the-art TPT methods on two backbone families. On the ViT side~\citep{dosovitskiy2020image}, we integrate $\mathtt{TPD}$ with VPT~\citep{jia2022visual}, E2VPT~\citep{han20232vpt}, VFPT~\citep{zeng2024visual}, and BPT~\citep{wang2025attention}, all using a supervised ViT-B/16 pretrained on ImageNet-21k~\citep{deng2009imagenet}. On the CLIP side~\citep{radford2021learning}, we evaluate representative text TPT baselines TPT \citep{shu2022test}, C-TPT \citep{yoon2024c}, R-TPT \citep{sheng2025r}, O-TPT \citep{sharifdeen2025tpt}, CAT \citep{zhang2025cat} and DyPt \citep{xiao2025dynaprompt} in both episodic and online modes.  $\mathtt{TPD}$ maintains only update-history statistics across samples and does not carry adapted prompt parameters across episodes. For each baseline, we report both the original result and its $\mathtt{TPD}$-augmented variant, denoted as ``Baseline~{+$\mathtt{TPD}$}''.

\paragraph{Implementation Details.}
We fix $\mathtt{TPD}$ hyperparameters across backbones. The local Hankel subspace uses a delay length of $K{=}3$ and $M{=}8$ Hankel columns, corresponding to a local window of $W{=}K+M-1{=}10$ past prompt updates. The effective rank $r$ is determined adaptively from the singular spectrum and satisfies $r\le \min(Kd,M)$. Before $W$ past updates are available, we use the base prompt update without $\mathtt{TPD}$ control. For $\mathtt{AKC}$, we set the stability threshold $\bar\rho{=}1.0$ and the minimum gain $g_{\min}{=}0.1$. For $\mathtt{HUR}$, we use $\beta_c{=}0.25$, $\beta_n{=}0.1$ on ViT and CLIP.

On the ViT branch, we follow the VPT~\citep{jia2022visual} protocol with $10$ deep prompt tokens. We perform one TTA step per test sample with learning rate $10^{-3}$ using SGD, a batch size of $128$, and select the top $10\%$ most confident samples for the entropy-minimization objective~\citep{wang2020tent}. On the CLIP branch, we build on CoOp~\citep{zhou2022learning} with $N{=}4$ learnable context tokens initialized from ``\texttt{a photo of a}''. 
In episodic mode, we follow the standard episodic protocol by resetting the prompt state after each episode. The only cross-episode state maintained by $\mathtt{TPD}$ is a lightweight history of recent prompt-update vectors used by the controller to estimate local trajectory stability. Importantly, $\mathtt{TPD}$ does not carry adapted prompt parameters, optimizer states, or episode-specific prompt features across episodes. This separates model adaptation from trajectory stabilization: each episode starts from the same reset prompt state as the base episodic method, while the controller uses past update statistics only to decide how to scale and route the current prompt update.
All experiments are run on an NVIDIA H100 GPU.

\subsection{Full Per-Dataset Results}

We provide full per-dataset results for the compressed comparisons and ablations reported in the main paper. Tables~A1--A4 expand the simple stabilization comparison in Table~\ref{tab:simple_stabilization}, covering update clipping, EMA update, norm control, low-pass filtering, and $\mathtt{TPD}$ on VPT and TPT under DG and CSD settings. Tables~A5--A10 expand the ablation study in Table~\ref{tab:ablation}, separately reporting Koopman operator choices, projection choices, and $\mathtt{HUR}$ routing variants. These tables show that the mean improvements in the main paper are supported by broad per-dataset gains rather than by a small number of outlier datasets.

Furthermore, VPT, E2VPT, and VFPT were originally developed as training-time visual prompt tuning methods rather than test-time adaptation algorithms. In our setting, we adapt them with the same entropy-minimization objective used in standard TPT, but this naive test-time variant has no mechanism for stabilizing prompt-update trajectories. As a result, visual-prompt updates can expand, oscillate, or drift under unlabeled target samples, and the resulting OOD mean may fall below the frozen ViT zero-shot baseline, e.g., $39.66\%$ versus $46.17\%$ in Table~\ref{tab:dg}. This does not indicate that visual prompt tuning itself is ineffective; rather, it shows that naive TTA on visual prompts remains unstable without trajectory-level control. $\mathtt{TPD}$ addresses this missing stabilizer by regulating update magnitude and refining update direction, turning unstable test-time tuning into effective adaptation that surpasses zero-shot under shifted-domain evaluation and remains competitive under cross-dataset transfer.

\subsubsection{Simple stabilization baselines.}
\label{app:full_per_dataset}
The simple stabilization baselines improve over the raw TPT updates on several datasets, but their gains are smaller and less consistent than $\mathtt{TPD}$. In particular, clipping and norm control mainly restrict update scale, while EMA and low-pass filtering smooth the trajectory without distinguishing persistent, oscillatory, and residual components. The per-dataset results below show that the trajectory-aware decomposition and routing in $\mathtt{TPD}$ provide more reliable improvements across both visual and text prompts.

\begingroup
\renewcommand{\thetable}{A1}
\begin{table*}[!htbp]
\caption{Per-dataset simple stabilization results for VPT under domain generalization. OoD Mean$_4$ averages ImageNet-V2, ImageNet-Sketch, ImageNet-A, and ImageNet-R.}
\label{tab:appendixA1}
\begin{center}\resizebox{0.98\textwidth}{!}{%
\begin{tabular}{l ccccc c}
\toprule
Variant & ImageNet & ImageNet-V2 & ImageNet-Sketch & ImageNet-A & ImageNet-R & OoD Mean$_{\mathrm{4}}$ \\
\midrule
Baseline & 78.16 & 64.55 & 32.77 & 19.17 & 42.14 & 39.66 \\
+ Update clipping & 78.84 & 65.52 & 34.09 & 22.95 & 45.11 & 41.92 \\
+ EMA update & 79.35 & 65.98 & 36.82 & 24.44 & 43.71 & 42.74 \\
+ Norm control & 79.76 & 66.66 & 34.26 & 25.60 & 48.09 & 43.65 \\
+ Low-pass filtering & 80.14 & 68.29 & 34.76 & 28.43 & 45.65 & 44.28 \\
VPT+$\mathtt{TPD}$ & 81.31 & 69.87 & 36.02 & 38.98 & 46.47 & \textbf{47.84} \\
\bottomrule
\end{tabular}%
}
\end{center}
\end{table*}
\endgroup
\FloatBarrier

\begingroup
\renewcommand{\thetable}{A2}
\begin{table*}[!htbp]
\caption{Per-dataset simple stabilization results for VPT under cross-dataset transfer. Mean averages the ten downstream datasets.}
\label{tab:appendixA2}
\begin{center}\resizebox{\textwidth}{!}{%
\begin{tabular}{l cccccccccc c}
\toprule
Variant & Caltech & Pets & Cars & Flowers & Food & Aircraft & SUN & DTD & EuroSAT & UCF & Mean \\
\midrule
Baseline & 92.65 & 86.89 & 60.68 & 64.86 & 85.14 & 21.21 & 59.92 & 38.19 & 20.91 & 59.28 & 58.97 \\
+ Update clipping & 92.94 & 87.35 & 61.27 & 65.16 & 85.85 & 21.93 & 60.57 & 38.48 & 21.78 & 60.03 & 59.54 \\
+ EMA update & 93.27 & 87.46 & 61.07 & 65.33 & 85.67 & 22.47 & 61.00 & 38.77 & 22.79 & 60.80 & 59.86 \\
+ Norm control & 93.10 & 87.71 & 63.33 & 66.75 & 86.59 & 26.40 & 60.40 & 40.54 & 24.25 & 61.19 & 61.03 \\
+ Low-pass filtering & 93.39 & 89.29 & 62.23 & 67.32 & 86.20 & 21.99 & 62.55 & 39.07 & 25.15 & 61.93 & 60.91 \\
VPT+$\mathtt{TPD}$ & 95.37 & 89.38 & 62.81 & 68.09 & 86.45 & 24.46 & 63.79 & 42.25 & 32.17 & 62.01 & \textbf{62.68} \\
\bottomrule
\end{tabular}%
}
\end{center}
\end{table*}
\endgroup
\FloatBarrier

\begingroup
\renewcommand{\thetable}{A3}
\begin{table*}[!htbp]
\caption{Per-dataset simple stabilization results for TPT under domain generalization. OOD Mean averages the four shifted variants excluding ImageNet.}
\label{tab:appendixA3}
\begin{center}\resizebox{0.98\textwidth}{!}{%
\begin{tabular}{l ccccc c}
\toprule
Variant & ImageNet & ImageNet-V2 & ImageNet-Sketch & ImageNet-A & ImageNet-R & OOD Mean \\
\midrule
Baseline & 68.98 & 63.45 & 47.94 & 54.77 & 77.06 & 60.81 \\
+ Update clipping & 69.90 & 64.12 & 48.40 & 55.09 & 77.84 & 61.36 \\
+ EMA update & 69.86 & 63.80 & 49.82 & 55.25 & 77.74 & 61.65 \\
+ Norm control & 70.11 & 63.64 & 50.05 & 55.80 & 77.39 & 61.72 \\
+ Low-pass filtering & 70.62 & 64.49 & 49.02 & 56.03 & 79.38 & 62.23 \\
TPT+$\mathtt{TPD}$ & 71.19 & 65.98 & 49.62 & 56.73 & 79.98 & \textbf{63.08} \\
\bottomrule
\end{tabular}%
}
\end{center}
\end{table*}
\endgroup
\FloatBarrier

\begingroup
\renewcommand{\thetable}{A4}
\begin{table*}[!htbp]
\caption{Per-dataset simple stabilization results for TPT under cross-dataset transfer. Mean averages the ten downstream datasets.}
\label{tab:appendixA4}
\begin{center}\resizebox{\textwidth}{!}{%
\begin{tabular}{l cccccccccc c}
\toprule
Variant & Caltech & Pets & Cars & Flowers & Food & Aircraft & SUN & DTD & EuroSAT & UCF & Mean \\
\midrule
Baseline & 94.16 & 87.79 & 66.87 & 68.98 & 84.67 & 24.78 & 65.50 & 47.75 & 42.44 & 68.04 & 65.10 \\
+ Update clipping & 94.48 & 88.06 & 67.12 & 69.63 & 84.88 & 25.17 & 65.59 & 48.52 & 42.78 & 68.49 & 65.47 \\
+ EMA update & 94.24 & 88.53 & 67.47 & 69.75 & 85.13 & 26.04 & 65.91 & 48.58 & 43.85 & 68.73 & 65.82 \\
+ Norm control & 94.60 & 88.77 & 67.84 & 68.98 & 85.26 & 25.80 & 65.77 & 49.00 & 42.93 & 68.49 & 65.74 \\
+ Low-pass filtering & 94.56 & 88.66 & 68.09 & 69.83 & 85.58 & 26.34 & 66.37 & 51.18 & 43.11 & 68.57 & 66.23 \\
TPT+$\mathtt{TPD}$ & 96.32 & 90.31 & 68.55 & 71.06 & 88.32 & 27.19 & 68.51 & 51.98 & 46.03 & 72.00 & \textbf{68.03} \\
\bottomrule
\end{tabular}%
}
\end{center}
\end{table*}
\endgroup
\FloatBarrier

\subsubsection{Full detailed ablations.}
\label{app:abla_per_dataset}
We next report per-dataset ablations for the main components of $\mathtt{TPD}$. The Koopman ablation compares full and diagonal local estimators; the projection ablation compares plain projection with the Hankel construction; and the routing ablation tests whether preserving persistent components while attenuating oscillatory and residual components is beneficial. Across datasets, the diagonal estimator, Hankel projection, and full routing rule give the most stable overall behavior.

For episodic TPT, prompts are reset after each episode. $\mathtt{TPD}$ does not carry adapted prompt parameters across episodes; it only maintains lightweight update-history statistics for trajectory analysis. If these update-history statistics are ablated and reset after every episode as well, no valid cross-step trajectory is available. In this case, $\mathtt{TPD}$ applies identity control, \textit{i.e.}, $g_t=1$ and $\mathbf u_t^{\mathrm{hur}}=\mathbf u_t$, so the resulting method is exactly the corresponding episodic baseline. Therefore, this ablation is already represented by the baseline rows.

\begingroup
\renewcommand{\thetable}{A5}
\begin{table*}[!htbp]
\caption{Per-dataset Koopman operator ablation for VPT. \textbf{(a)} Domain generalization.}
\label{app:A5:dg}
\begin{center}\resizebox{0.95\textwidth}{!}{%
\begin{tabular}{l ccccc c}
\toprule
Variant & ImageNet & ImageNet-V2 & ImageNet-Sketch & ImageNet-A & ImageNet-R & OoD$_{4}$ \\
\midrule
Baseline & 78.16 & 64.55 & 32.77 & 19.17 & 42.14 & 39.66 \\
Full Koopman & 79.74 & 69.39 & 35.83 & 23.73 & 45.36 & 43.58 \\
VPT+$\mathtt{AKC}$ (diag.) & 80.94 & 69.70 & 38.33 & 28.47 & 47.17 & 45.92 \\
VPT+$\mathtt{TPD}$ & 81.31 & 69.87 & 36.02 & 38.98 & 46.47 & \textbf{47.84} \\
\bottomrule
\end{tabular}%
}
\end{center}
\end{table*}
\endgroup
\FloatBarrier

\begingroup
\renewcommand{\thetable}{A5}
\begin{table*}[!htbp]
\caption{Per-dataset Koopman operator ablation for VPT. \textbf{(b)} Cross-dataset transfer.}
\label{app:A5:csd}
\begin{center}\resizebox{\textwidth}{!}{%
\begin{tabular}{l cccccccccc c}
\toprule
Variant & Caltech & Pets & Cars & Flowers & Food & Aircraft & SUN & DTD & EuroSAT & UCF & Mean \\
\midrule
Baseline & 92.65 & 86.89 & 60.68 & 64.86 & 85.14 & 21.21 & 59.92 & 38.19 & 20.91 & 59.28 & 58.97 \\
Full Koopman & 93.67 & 88.14 & 61.30 & 66.34 & 85.88 & 24.63 & 62.44 & 39.78 & 24.84 & 61.56 & 60.86 \\
VPT+$\mathtt{AKC}$ (diag.) & 94.24 & 89.78 & 62.33 & 66.83 & 86.03 & 25.23 & 62.15 & 40.96 & 26.95 & 62.94 & 61.74 \\
VPT+$\mathtt{TPD}$ & 95.37 & 89.38 & 62.81 & 68.09 & 86.45 & 24.46 & 63.79 & 42.25 & 32.17 & 62.01 & \textbf{62.68} \\
\bottomrule
\end{tabular}%
}
\end{center}
\end{table*}
\endgroup
\FloatBarrier

\begingroup
\renewcommand{\thetable}{A6}
\begin{table*}[!htbp]
\caption{Per-dataset Koopman operator ablation for TPT. \textbf{(a)} Domain generalization.}
\label{app:A6:dg}
\begin{center}\resizebox{0.95\textwidth}{!}{%
\begin{tabular}{l ccccc c}
\toprule
Variant & ImageNet & ImageNet-V2 & ImageNet-Sketch & ImageNet-A & ImageNet-R & OOD Mean \\
\midrule
Baseline & 68.98 & 63.45 & 47.94 & 54.77 & 77.06 & 60.81 \\
Full Koopman & 69.86 & 63.73 & 49.87 & 55.83 & 77.76 & 61.80 \\
TPT+$\mathtt{AKC}$ (diag.) & 71.28 & 64.36 & 50.52 & 56.28 & 79.09 & 62.56 \\
TPT+$\mathtt{TPD}$ & 71.19 & 65.98 & 49.62 & 56.73 & 79.98 & \textbf{63.08} \\
\bottomrule
\end{tabular}%
}
\end{center}
\end{table*}
\endgroup
\FloatBarrier

\begingroup
\renewcommand{\thetable}{A6}
\begin{table*}[!htbp]
\caption{Per-dataset Koopman operator ablation for TPT. \textbf{(b)} Cross-dataset transfer.}
\label{app:A6:csd}
\begin{center}\resizebox{\textwidth}{!}{%
\begin{tabular}{l cccccccccc c}
\toprule
Variant & Caltech & Pets & Cars & Flowers & Food & Aircraft & SUN & DTD & EuroSAT & UCF & Mean \\
\midrule
Baseline & 94.16 & 87.79 & 66.87 & 68.98 & 84.67 & 24.78 & 65.50 & 47.75 & 42.44 & 68.04 & 65.10 \\
Full Koopman & 95.29 & 88.91 & 68.05 & 70.44 & 85.86 & 26.52 & 66.46 & 48.17 & 44.65 & 69.47 & 66.38 \\
TPT+$\mathtt{AKC}$ (diag.) & 95.42 & 89.18 & 68.77 & 70.60 & 88.29 & 26.43 & 68.19 & 50.30 & 44.95 & 69.97 & 67.21 \\
TPT+$\mathtt{TPD}$ & 96.32 & 90.31 & 68.55 & 71.06 & 88.32 & 27.19 & 68.51 & 51.98 & 46.03 & 72.00 & \textbf{68.03} \\
\bottomrule
\end{tabular}%
}
\end{center}
\end{table*}
\endgroup
\FloatBarrier

\begingroup
\renewcommand{\thetable}{A7}
\begin{table*}[!htbp]
\caption{Per-dataset projection ablation for VPT. \textbf{(a)} Domain generalization.}
\label{app:A7:dg}
\begin{center}\resizebox{0.95\textwidth}{!}{%
\begin{tabular}{l ccccc c}
\toprule
Variant & ImageNet & ImageNet-V2 & ImageNet-Sketch & ImageNet-A & ImageNet-R & OoD$_{4}$ \\
\midrule
Baseline & 78.16 & 64.55 & 32.77 & 19.17 & 42.14 & 39.66 \\
Plain projection & 79.17 & 67.46 & 34.53 & 23.41 & 43.83 & 42.31 \\
VPT+$\mathtt{HUR}$ (Hankel) & 81.10 & 69.84 & 38.60 & 26.25 & 49.21 & 45.98 \\
VPT+$\mathtt{TPD}$ & 81.31 & 69.87 & 36.02 & 38.98 & 46.47 & \textbf{47.84} \\
\bottomrule
\end{tabular}%
}
\end{center}
\end{table*}
\endgroup
\FloatBarrier

\begingroup
\renewcommand{\thetable}{A7}
\begin{table*}[!htbp]
\caption{Per-dataset projection ablation for VPT. \textbf{(b)} Cross-dataset transfer.}
\label{app:A7:csd}
\begin{center}\resizebox{\textwidth}{!}{%
\begin{tabular}{l cccccccccc c}
\toprule
Variant & Caltech & Pets & Cars & Flowers & Food & Aircraft & SUN & DTD & EuroSAT & UCF & Mean \\
\midrule
Baseline & 92.65 & 86.89 & 60.68 & 64.86 & 85.14 & 21.21 & 59.92 & 38.19 & 20.91 & 59.28 & 58.97 \\
Plain projection & 93.87 & 86.92 & 61.57 & 65.41 & 85.84 & 23.01 & 61.29 & 39.95 & 24.33 & 60.03 & 60.22 \\
VPT+$\mathtt{HUR}$ (Hankel) & 95.21 & 89.26 & 64.99 & 67.07 & 86.85 & 25.26 & 64.69 & 40.37 & 26.81 & 61.30 & 62.18 \\
VPT+$\mathtt{TPD}$ & 95.37 & 89.38 & 62.81 & 68.09 & 86.45 & 24.46 & 63.79 & 42.25 & 32.17 & 62.01 & \textbf{62.68} \\
\bottomrule
\end{tabular}%
}
\end{center}
\end{table*}
\endgroup
\FloatBarrier

\begingroup
\renewcommand{\thetable}{A8}
\begin{table*}[!htbp]
\caption{Per-dataset projection ablation for TPT. \textbf{(a)} Domain generalization.}
\label{app:A8:dg}
\begin{center}\resizebox{0.95\textwidth}{!}{%
\begin{tabular}{l ccccc c}
\toprule
Variant & ImageNet & ImageNet-V2 & ImageNet-Sketch & ImageNet-A & ImageNet-R & OOD Mean \\
\midrule
Baseline & 68.98 & 63.45 & 47.94 & 54.77 & 77.06 & 60.81 \\
Plain projection & 69.46 & 64.31 & 48.65 & 54.99 & 77.78 & 61.43 \\
TPT+$\mathtt{HUR}$ (Hankel) & 70.77 & 65.31 & 50.24 & 56.77 & 79.05 & 62.84 \\
TPT+$\mathtt{TPD}$ & 71.19 & 65.98 & 49.62 & 56.73 & 79.98 & \textbf{63.08} \\
\bottomrule
\end{tabular}%
}
\end{center}
\end{table*}
\endgroup
\FloatBarrier

\begingroup
\renewcommand{\thetable}{A8}
\begin{table*}[!htbp]
\caption{Per-dataset projection ablation for TPT. \textbf{(b)} Cross-dataset transfer.}
\label{app:A8:csd}
\begin{center}\resizebox{\textwidth}{!}{%
\begin{tabular}{l cccccccccc c}
\toprule
Variant & Caltech & Pets & Cars & Flowers & Food & Aircraft & SUN & DTD & EuroSAT & UCF & Mean \\
\midrule
Baseline & 94.16 & 87.79 & 66.87 & 68.98 & 84.67 & 24.78 & 65.50 & 47.75 & 42.44 & 68.04 & 65.10 \\
Plain projection & 94.77 & 88.17 & 67.42 & 70.20 & 85.43 & 26.04 & 66.19 & 48.76 & 43.56 & 69.52 & 66.01 \\
TPT+$\mathtt{HUR}$ (Hankel) & 95.25 & 89.07 & 67.98 & 71.74 & 86.68 & 27.18 & 68.60 & 53.01 & 45.59 & 70.34 & 67.54 \\
TPT+$\mathtt{TPD}$ & 96.32 & 90.31 & 68.55 & 71.06 & 88.32 & 27.19 & 68.51 & 51.98 & 46.03 & 72.00 & \textbf{68.03} \\
\bottomrule
\end{tabular}%
}
\end{center}
\end{table*}
\endgroup
\FloatBarrier

\begingroup
\renewcommand{\thetable}{A9}
\begin{table*}[!htbp]
\caption{Per-dataset $\mathtt{HUR}$ routing ablation for VPT. \textbf{(a)} Domain generalization.}
\label{app:A9:dg}
\begin{center}\resizebox{0.95\textwidth}{!}{%
\begin{tabular}{l ccccc c}
\toprule
Variant & ImageNet & ImageNet-V2 & ImageNet-Sketch & ImageNet-A & ImageNet-R & OoD$_{4}$ \\
\midrule
Baseline & 78.16 & 64.55 & 32.77 & 19.17 & 42.14 & 39.66 \\
w/o Pres. & 81.33 & 69.32 & 35.99 & 36.51 & 46.41 & 47.06 \\
w/o Osc. & 81.36 & 69.38 & 36.48 & 36.65 & 46.61 & 47.28 \\
w/o Res. & 81.02 & 69.08 & 35.25 & 36.35 & 46.68 & 46.84 \\
VPT+$\mathtt{TPD}$ & 81.31 & 69.87 & 36.02 & 38.98 & 46.47 & \textbf{47.84} \\
\bottomrule
\end{tabular}%
}
\end{center}
\end{table*}
\endgroup
\FloatBarrier

\begingroup
\renewcommand{\thetable}{A9}
\begin{table*}[!htbp]
\caption{Per-dataset $\mathtt{HUR}$ routing ablation for VPT. \textbf{(b)} Cross-dataset transfer.}
\label{app:A9:csd}
\begin{center}\resizebox{\textwidth}{!}{%
\begin{tabular}{l cccccccccc c}
\toprule
Variant & Caltech & Pets & Cars & Flowers & Food & Aircraft & SUN & DTD & EuroSAT & UCF & Mean \\
\midrule
Baseline & 92.65 & 86.89 & 60.68 & 64.86 & 85.14 & 21.21 & 59.92 & 38.19 & 20.91 & 59.28 & 58.97 \\
w/o Pres. & 95.42 & 88.96 & 62.50 & 67.97 & 85.98 & 24.54 & 63.06 & 41.19 & 31.81 & 61.49 & 62.29 \\
w/o Osc. & 95.66 & 88.09 & 62.19 & 68.21 & 86.34 & 24.00 & 61.89 & 41.78 & 33.26 & 62.28 & 62.37 \\
w/o Res. & 94.69 & 88.96 & 62.16 & 67.84 & 86.01 & 23.85 & 63.08 & 43.03 & 30.81 & 61.64 & 62.21 \\
VPT+$\mathtt{TPD}$ & 95.37 & 89.38 & 62.81 & 68.09 & 86.45 & 24.46 & 63.79 & 42.25 & 32.17 & 62.01 & \textbf{62.68} \\
\bottomrule
\end{tabular}%
}
\end{center}
\end{table*}
\endgroup
\FloatBarrier

\begingroup
\renewcommand{\thetable}{A10}
\begin{table*}[!htbp]
\caption{Per-dataset $\mathtt{HUR}$ routing ablation for TPT. \textbf{(a)} Domain generalization.}
\label{app:A10:dg}
\begin{center}\resizebox{0.95\textwidth}{!}{%
\begin{tabular}{l ccccc c}
\toprule
Variant & ImageNet & ImageNet-V2 & ImageNet-Sketch & ImageNet-A & ImageNet-R & OOD Mean \\
\midrule
Baseline & 68.98 & 63.45 & 47.94 & 54.77 & 77.06 & 60.81 \\
w/o Pres. & 70.54 & 65.86 & 49.05 & 56.71 & 80.17 & 62.95 \\
w/o Osc. & 71.25 & 65.58 & 49.42 & 56.36 & 80.00 & 62.84 \\
w/o Res. & 71.70 & 66.02 & 49.22 & 56.36 & 79.61 & 62.80 \\
TPT+$\mathtt{TPD}$ & 71.19 & 65.98 & 49.62 & 56.73 & 79.98 & \textbf{63.08} \\
\bottomrule
\end{tabular}%
}
\end{center}
\end{table*}
\endgroup
\FloatBarrier

\begingroup
\renewcommand{\thetable}{A10}
\begin{table*}[!htbp]
\caption{Per-dataset $\mathtt{HUR}$ routing ablation for TPT. \textbf{(b)} Cross-dataset transfer.}
\label{app:A10:csd}
\begin{center}\resizebox{\textwidth}{!}{%
\begin{tabular}{l cccccccccc c}
\toprule
Variant & Caltech & Pets & Cars & Flowers & Food & Aircraft & SUN & DTD & EuroSAT & UCF & Mean \\
\midrule
Baseline & 94.16 & 87.79 & 66.87 & 68.98 & 84.67 & 24.78 & 65.50 & 47.75 & 42.44 & 68.04 & 65.10 \\
w/o Pres. & 96.27 & 90.35 & 68.00 & 70.65 & 88.08 & 27.06 & 68.88 & 51.65 & 45.48 & 70.90 & 67.73 \\
w/o Osc. & 96.92 & 89.89 & 68.10 & 70.36 & 89.18 & 26.16 & 68.99 & 52.01 & 46.30 & 70.82 & 67.87 \\
w/o Res. & 96.47 & 89.62 & 68.19 & 70.24 & 87.38 & 27.48 & 67.96 & 51.24 & 46.06 & 70.79 & 67.54 \\
TPT+$\mathtt{TPD}$ & 96.32 & 90.31 & 68.55 & 71.06 & 88.32 & 27.19 & 68.51 & 51.98 & 46.03 & 72.00 & \textbf{68.03} \\
\bottomrule
\end{tabular}%
}
\end{center}
\end{table*}
\endgroup
\FloatBarrier

\subsection{Hyperparameter Sensitivity}
\label{app:hyper_sensitivity}
We provide per-dataset sensitivity analyses for the three groups of $\mathtt{TPD}$ hyperparameters: the Hankel window $(K,M)$, the $\mathtt{AKC}$ control parameters $(\bar\rho,g_{\min})$, and the $\mathtt{HUR}$ routing weights $(\beta_c,\beta_n)$. All heatmaps are normalized within each panel for visualization, so warmer colors indicate better performance for the same dataset and baseline. Overall, the high-performing regions are broad rather than isolated, showing that $\mathtt{TPD}$ is not sensitive to a single carefully tuned setting.
All main-table results use a single fixed $\mathtt{TPD}$ configuration across datasets, prompt modalities, backbone families, and baselines. The sensitivity sweeps in this section are diagnostic analyses for robustness, not per-dataset or per-baseline hyperparameter selection. Thus, the reported gains do not rely on choosing target-specific controller parameters.

\paragraph{Hankel window.}
We first sweep the Hankel delay length $K$ and the number of Hankel columns $M$. Across both VPT and TPT, very short histories provide insufficient temporal context, while overly long histories may mix outdated updates into the local model. In contrast, short-to-moderate windows consistently form a stable high-performing region across both DG and CSD datasets. The default setting $K=3$ and $M=8$ lies inside this region, supporting its use as a fixed configuration across prompt types and datasets.

\begin{figure*}[!htbp]
\centering
\begin{minipage}{0.49\linewidth}
\centering
\includegraphics[width=\linewidth]{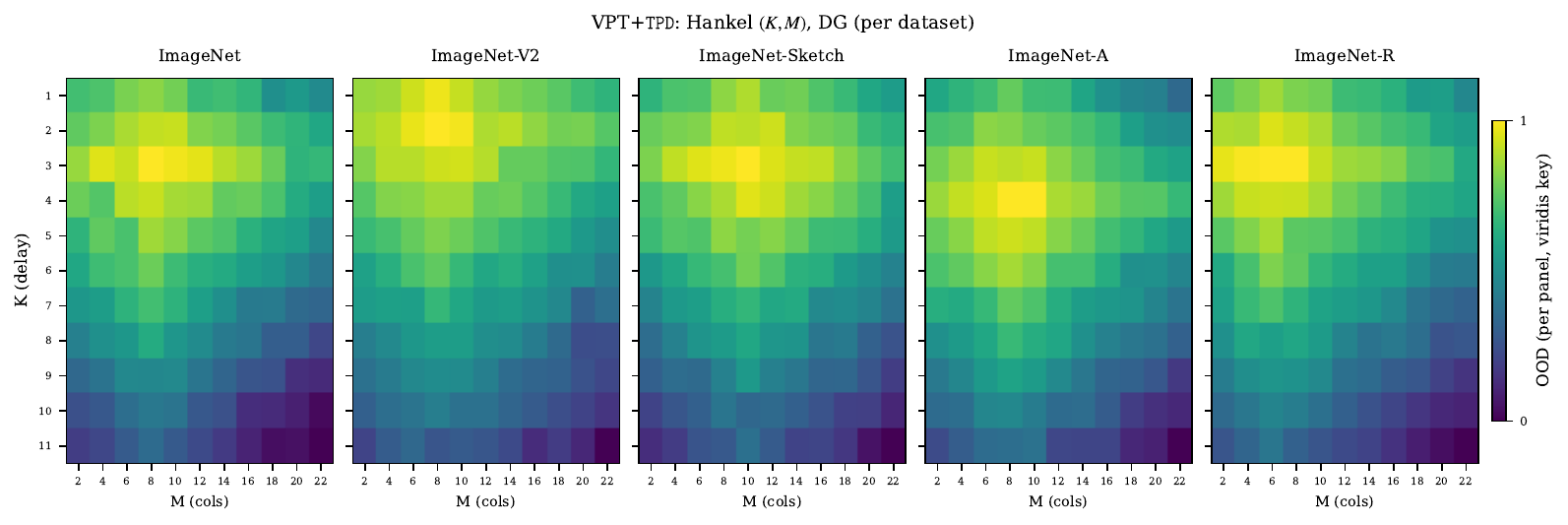}\\
\small (a) VPT+$\mathtt{TPD}$, domain generalization
\end{minipage}\hfill
\begin{minipage}{0.49\linewidth}
\centering
\includegraphics[width=\linewidth]{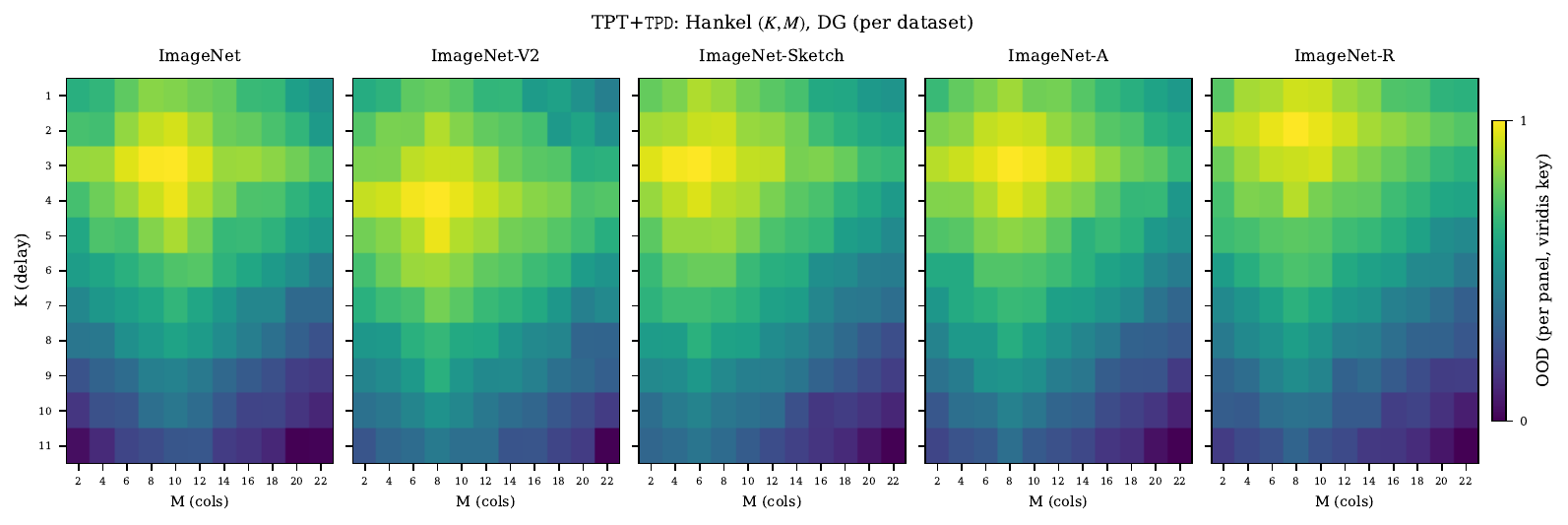}\\
\small (b) TPT+$\mathtt{TPD}$, domain generalization
\end{minipage}
\caption{Sensitivity to the Hankel window on domain-generalization datasets. Each heatmap reports per-dataset performance of VPT+$\mathtt{TPD}$ or TPT+$\mathtt{TPD}$ while sweeping the delay length $K$ and the number of Hankel columns $M$. Colors are normalized within each panel, with warmer colors indicating higher accuracy. Both prompt types show a broad high-performing region at short-to-moderate windows, supporting the default $K=3,M=8$.}
\label{fig:appendixA1perDG}
\end{figure*}

\begin{figure*}[!htbp]
\centering
\begin{minipage}{0.49\linewidth}
\centering
\includegraphics[width=\linewidth]{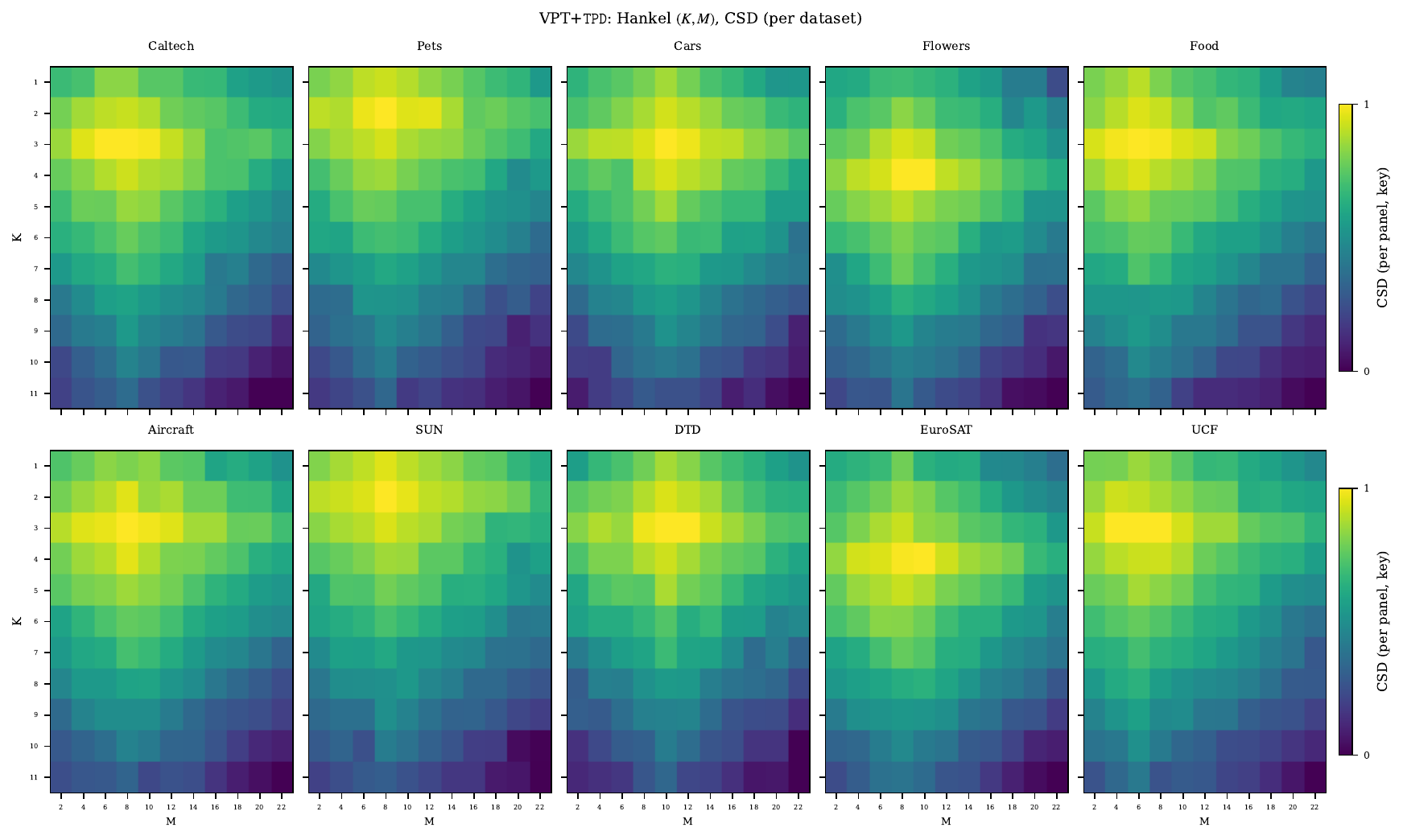}\\
\small (a) VPT+$\mathtt{TPD}$, CSD
\end{minipage}\hfill
\begin{minipage}{0.49\linewidth}
\centering
\includegraphics[width=\linewidth]{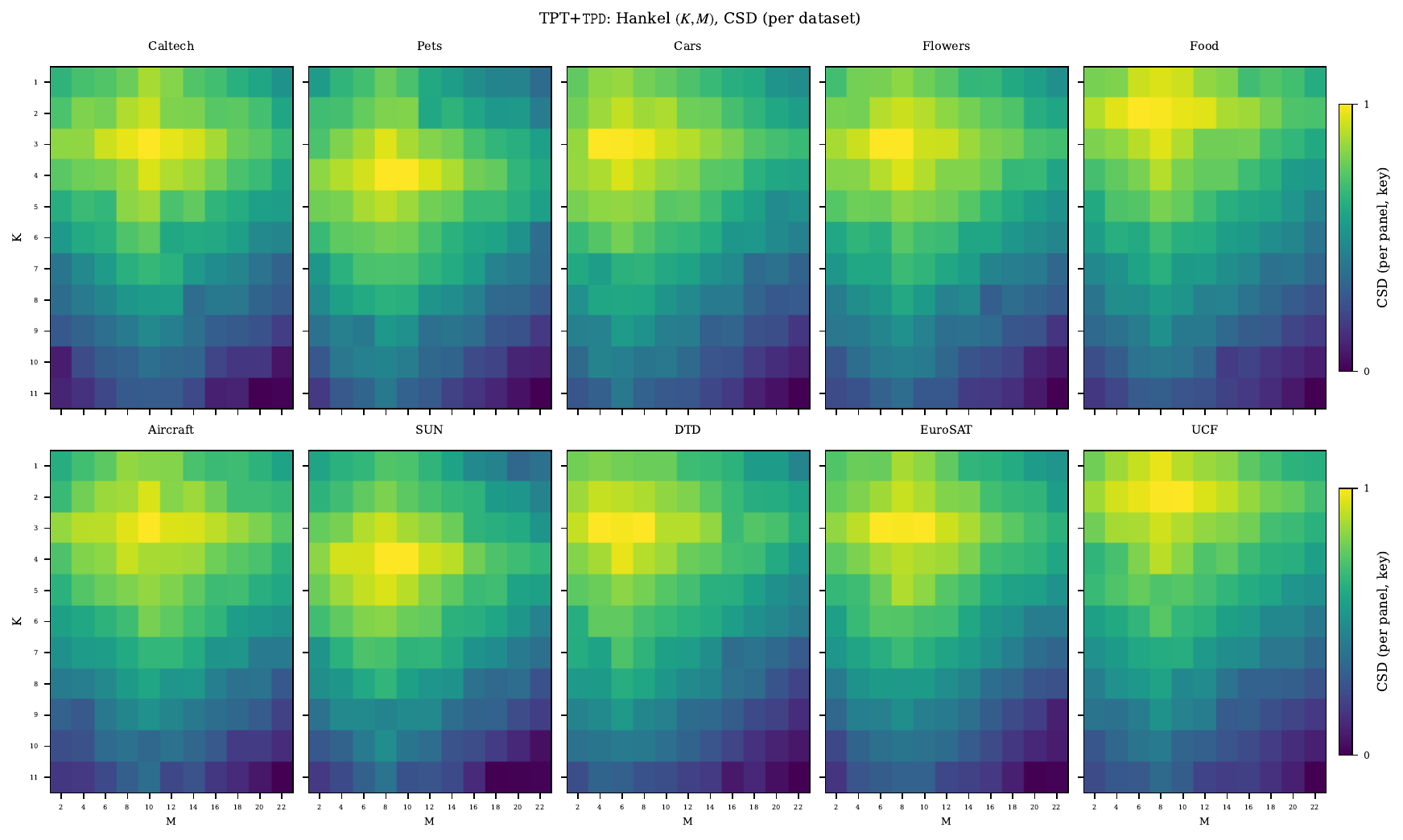}\\
\small (b) TPT+$\mathtt{TPD}$, CSD
\end{minipage}
\caption{Sensitivity to the Hankel window on cross-dataset transfer. The same $(K,M)$ sweep is evaluated on ten downstream datasets. The high-performing region remains broad across datasets and prompt types, while excessively short or long histories are less reliable.}
\label{fig:appendixA1perCSD}
\end{figure*}

\paragraph{AKC threshold and gain floor.}
We next sweep the $\mathtt{AKC}$ stability threshold $\bar\rho$ and the minimum gain $g_{\min}$. The results show a consistent plateau around $\bar\rho=1.0$ with a small gain floor. This matches the intended role of $\mathtt{AKC}$: the controller should intervene when the local Koopman dynamics become expansive, but it should not over-dampen stable updates. Large $g_{\min}$ values weaken this intervention by preventing sufficient shrinkage, while overly strict thresholds can suppress useful updates. The default $\bar\rho=1.0$ and $g_{\min}=0.1$ balance these two effects.

\begin{figure*}[!htbp]
\centering
\begin{minipage}{0.49\linewidth}
\centering
\includegraphics[width=\linewidth]{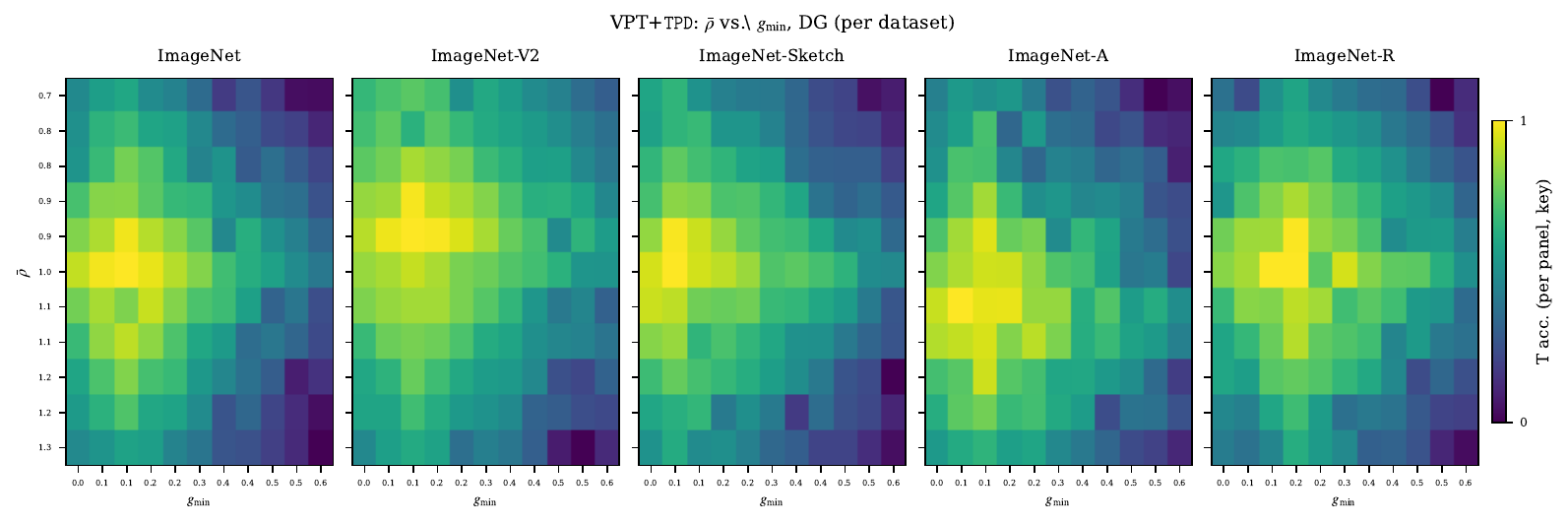}\\
\small (a) VPT+$\mathtt{TPD}$, domain generalization
\end{minipage}\hfill
\begin{minipage}{0.49\linewidth}
\centering
\includegraphics[width=\linewidth]{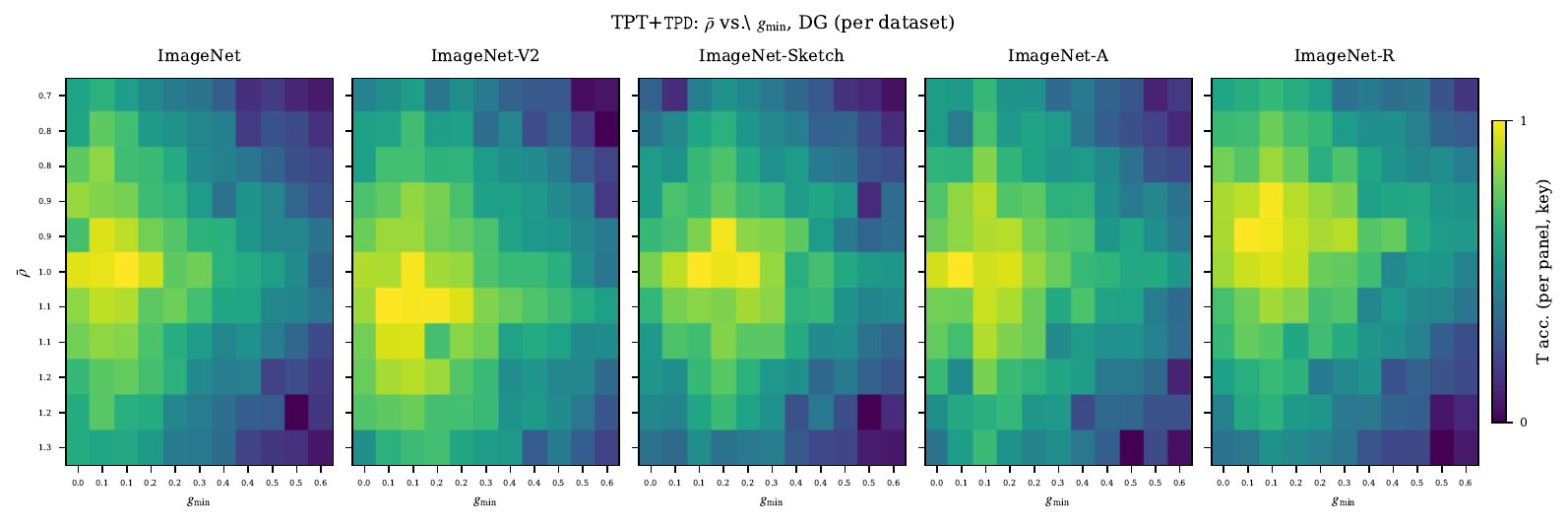}\\
\small (b) TPT+$\mathtt{TPD}$, domain generalization
\end{minipage}
\caption{Sensitivity of $\mathtt{AKC}$ on domain-generalization datasets. Each heatmap sweeps the spectral-radius threshold $\bar\rho$ and gain floor $g_{\min}$ for VPT+$\mathtt{TPD}$ or TPT+$\mathtt{TPD}$. Colors are normalized within each panel. Performance is strongest in a broad region around $\bar\rho=1.0$ with a small gain floor, indicating that $\mathtt{AKC}$ benefits from suppressing locally expansive dynamics without over-damping stable updates.}
\label{fig:appendixA2perDG}
\end{figure*}

\begin{figure*}[!htbp]
\centering
\begin{minipage}{0.49\linewidth}
\centering
\includegraphics[width=\linewidth]{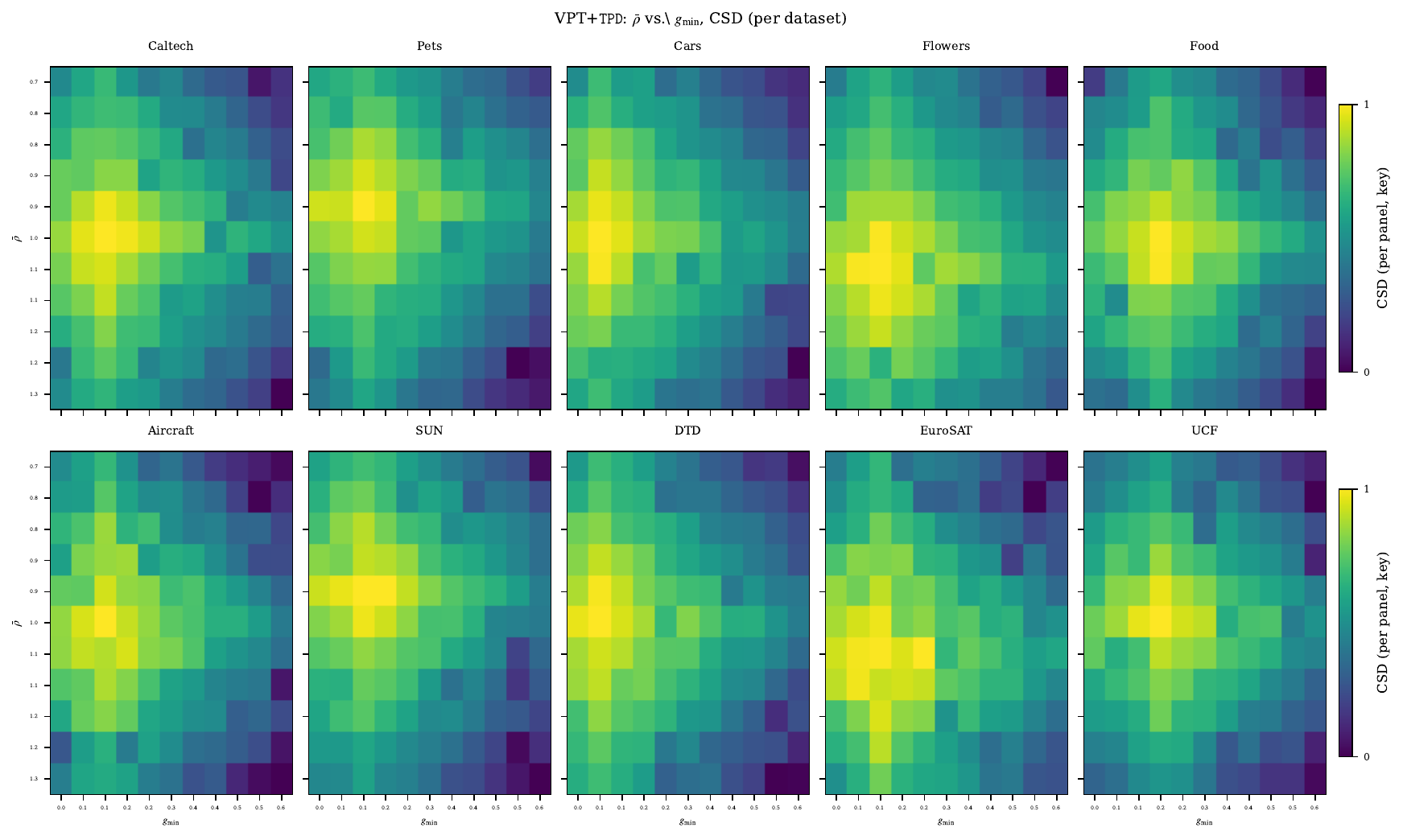}\\
\small (a) VPT+$\mathtt{TPD}$, CSD
\end{minipage}\hfill
\begin{minipage}{0.49\linewidth}
\centering
\includegraphics[width=\linewidth]{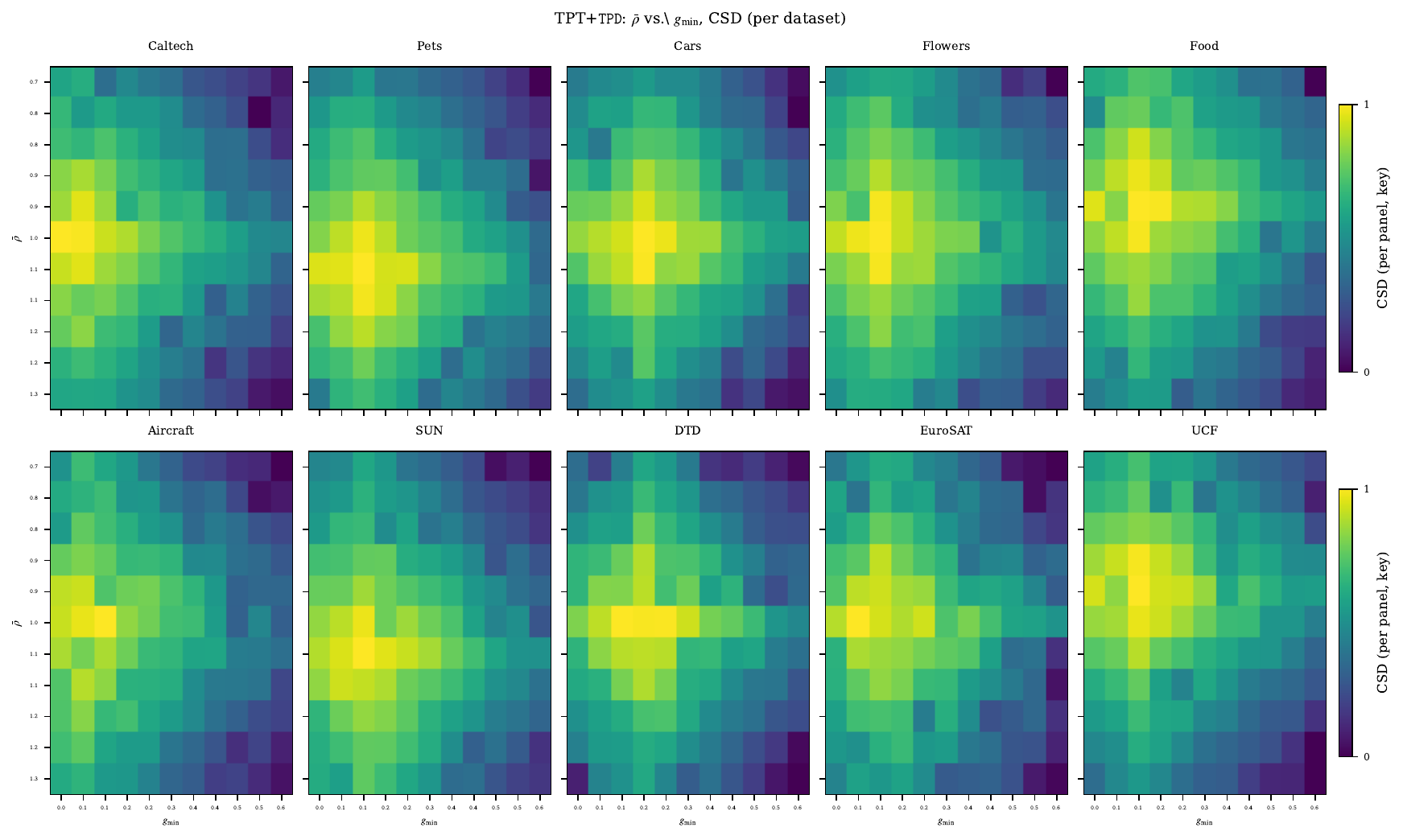}\\
\small (b) TPT+$\mathtt{TPD}$, CSD
\end{minipage}
\caption{Sensitivity of $\mathtt{AKC}$ on cross-dataset transfer. The $\bar\rho$--$g_{\min}$ sweep shows that moderate stability control is robust across datasets, whereas a large gain floor weakens the controller and overly strict settings may suppress useful adaptation.}
\label{fig:appendixA2perCSD}
\end{figure*}

\paragraph{HUR routing weights.}
Finally, we sweep the $\mathtt{HUR}$ routing weights under the constraint $0\le\beta_n\le\beta_c\le1$. The best regions favor strong suppression of residual components and moderate attenuation of oscillatory components. This agrees with the design of $\mathtt{HUR}$: persistent components are preserved with unit weight, oscillatory components are attenuated rather than discarded, and poorly explained residual components are suppressed more strongly. Performance degrades when both weights approach one, where $\mathtt{HUR}$ becomes close to the raw update and loses its selective-routing effect.

\begin{figure*}[!htbp]
\centering
\begin{minipage}{0.49\linewidth}
\centering
\includegraphics[width=\linewidth]{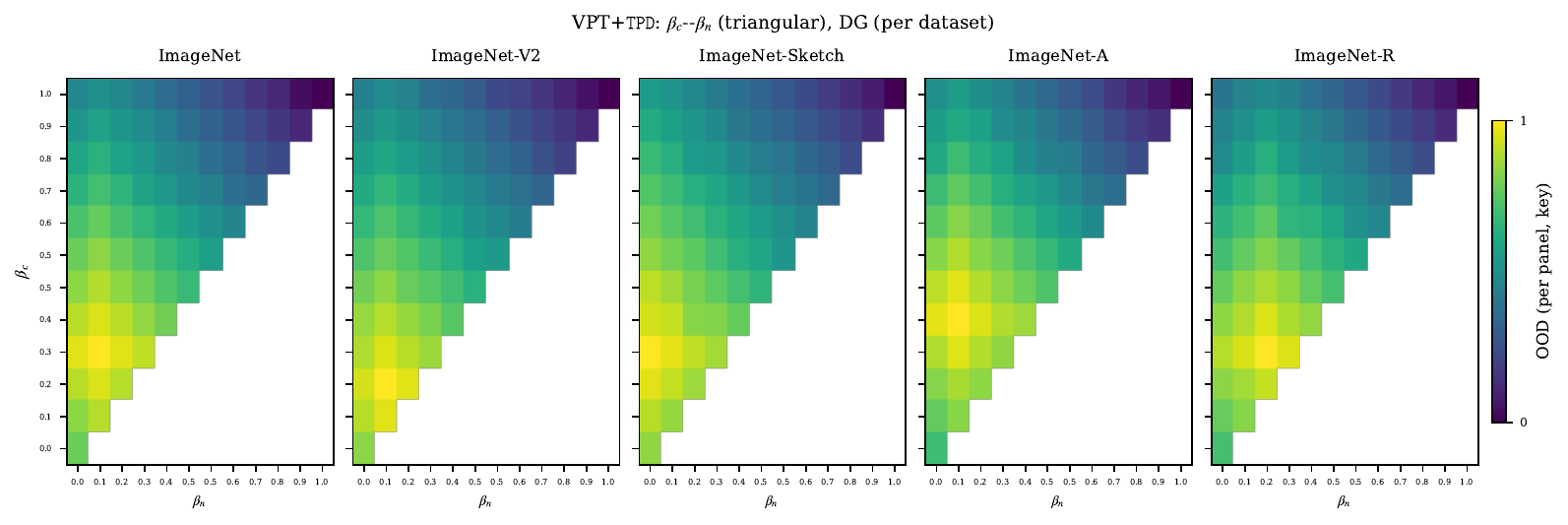}\\
\small (a) VPT+$\mathtt{TPD}$, domain generalization
\end{minipage}\hfill
\begin{minipage}{0.49\linewidth}
\centering
\includegraphics[width=\linewidth]{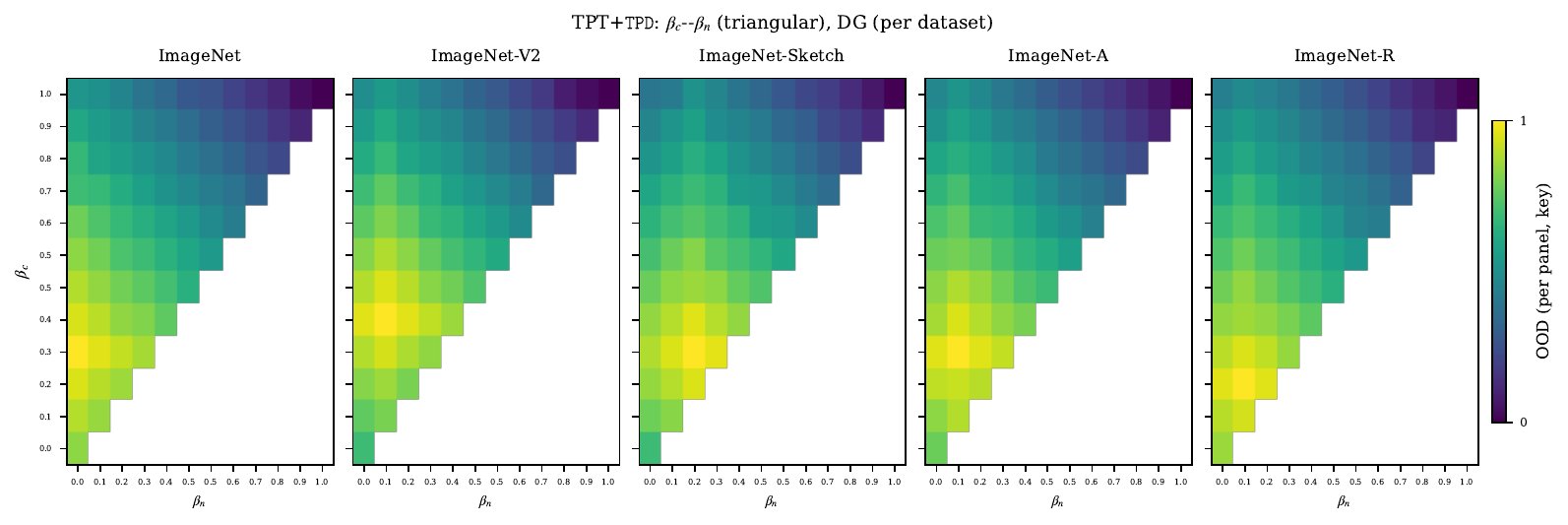}\\
\small (b) TPT+$\mathtt{TPD}$, domain generalization
\end{minipage}
\caption{Sensitivity of $\mathtt{HUR}$ routing weights on domain-generalization datasets. Each heatmap sweeps the triangular grid $0\le\beta_n\le\beta_c\le1$, where $\beta_c$ controls the oscillatory component and $\beta_n$ controls the residual component; the masked region has $\beta_n>\beta_c$. Warmer colors indicate higher per-panel accuracy. The high-performing region favors small residual weights and moderate oscillatory attenuation, supporting selective routing rather than using the raw update unchanged.}
\label{fig:appendixA3perDG}
\end{figure*}

\begin{figure*}[!htbp]
\centering
\begin{minipage}{0.49\linewidth}
\centering
\includegraphics[width=\linewidth]{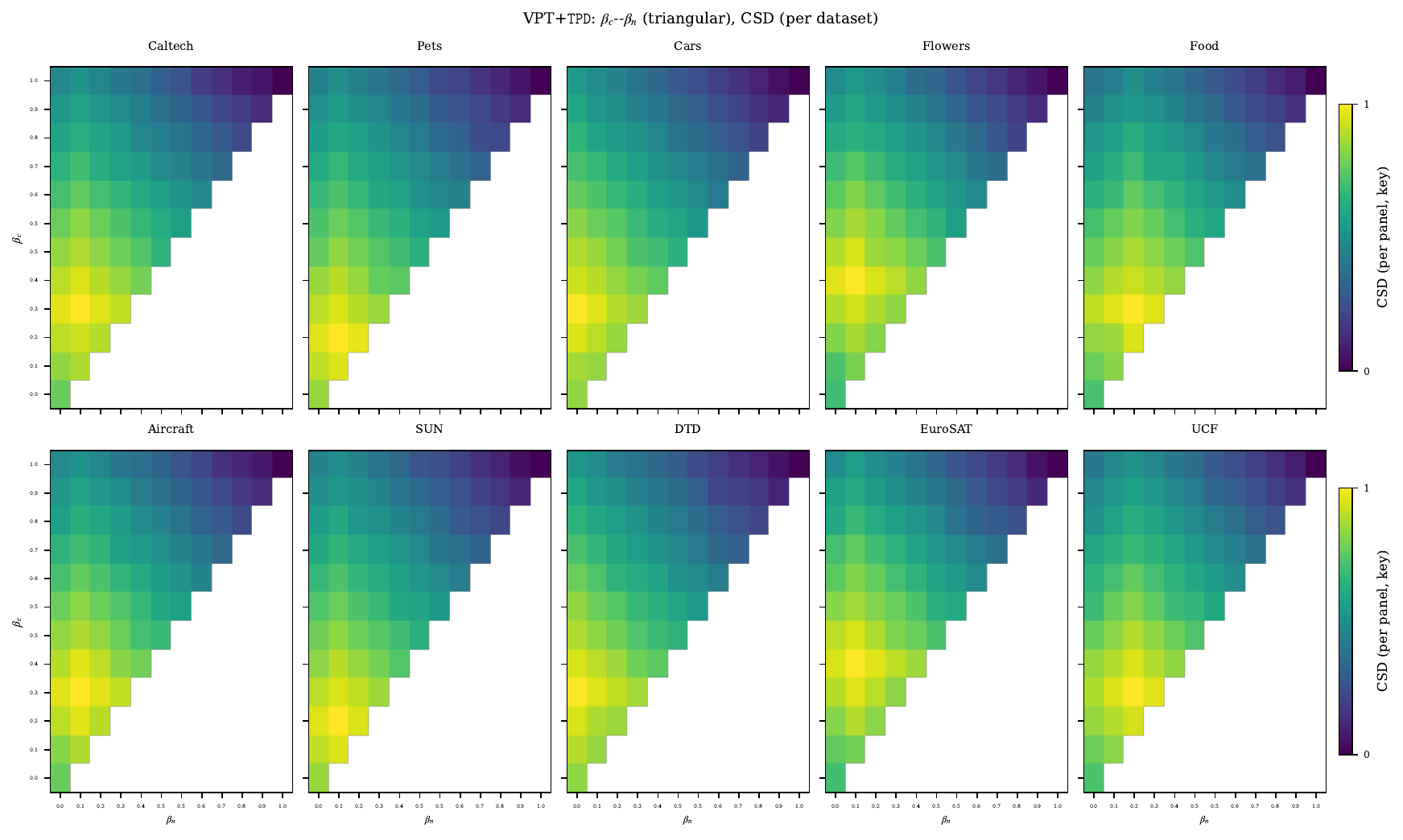}\\
\small (a) VPT+$\mathtt{TPD}$, CSD
\end{minipage}\hfill
\begin{minipage}{0.49\linewidth}
\centering
\includegraphics[width=\linewidth]{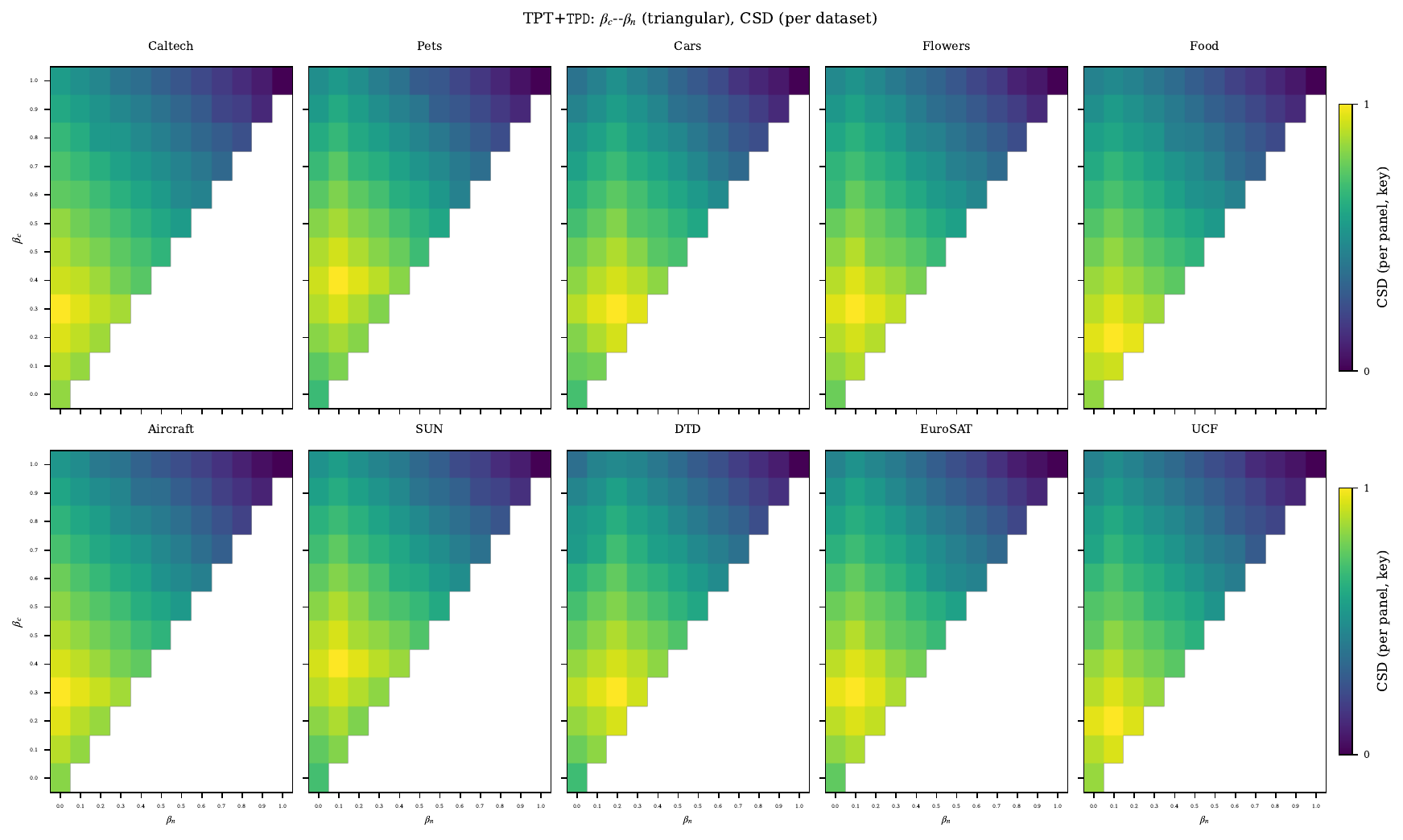}\\
\small (b) TPT+$\mathtt{TPD}$, CSD
\end{minipage}
\caption{Sensitivity of $\mathtt{HUR}$ routing weights on cross-dataset transfer. The triangular $\beta_c$--$\beta_n$ sweep is consistent across VPT+$\mathtt{TPD}$ and TPT+$\mathtt{TPD}$: retaining persistent components while attenuating oscillatory and residual components yields a broad high-performing region, whereas weights near one reduce $\mathtt{HUR}$ to weak routing and degrade performance.}
\label{fig:appendixA3perCSD}
\end{figure*}

\subsection{Persistent Ratio and Adaptation Gain}
\label{app:persistent_gain}
We further examine the link between persistent ratio and adaptation gain. The persistent ratio measures the fraction of a prompt update aligned with locally consistent directions in the recent update trajectory. A larger value indicates more stable adaptation content and less oscillatory or residual content.

For each dataset, we plot the persistent ratio against the accuracy gain obtained by $\mathtt{TPD}$. The red dashed curve is a piecewise linear fitted trend line over the persistent-ratio axis. Across TPT, the gain generally increases from low to moderate persistent ratios and reaches a peak at an intermediate range, after which the trend saturates or slightly decreases on several datasets. This pattern suggests that persistent ratio is a useful trajectory-level diagnostic, while excessively persistent updates may no longer bring additional adaptation benefit.

\begin{figure*}[!htbp]
\centering
\includegraphics[width=\linewidth]{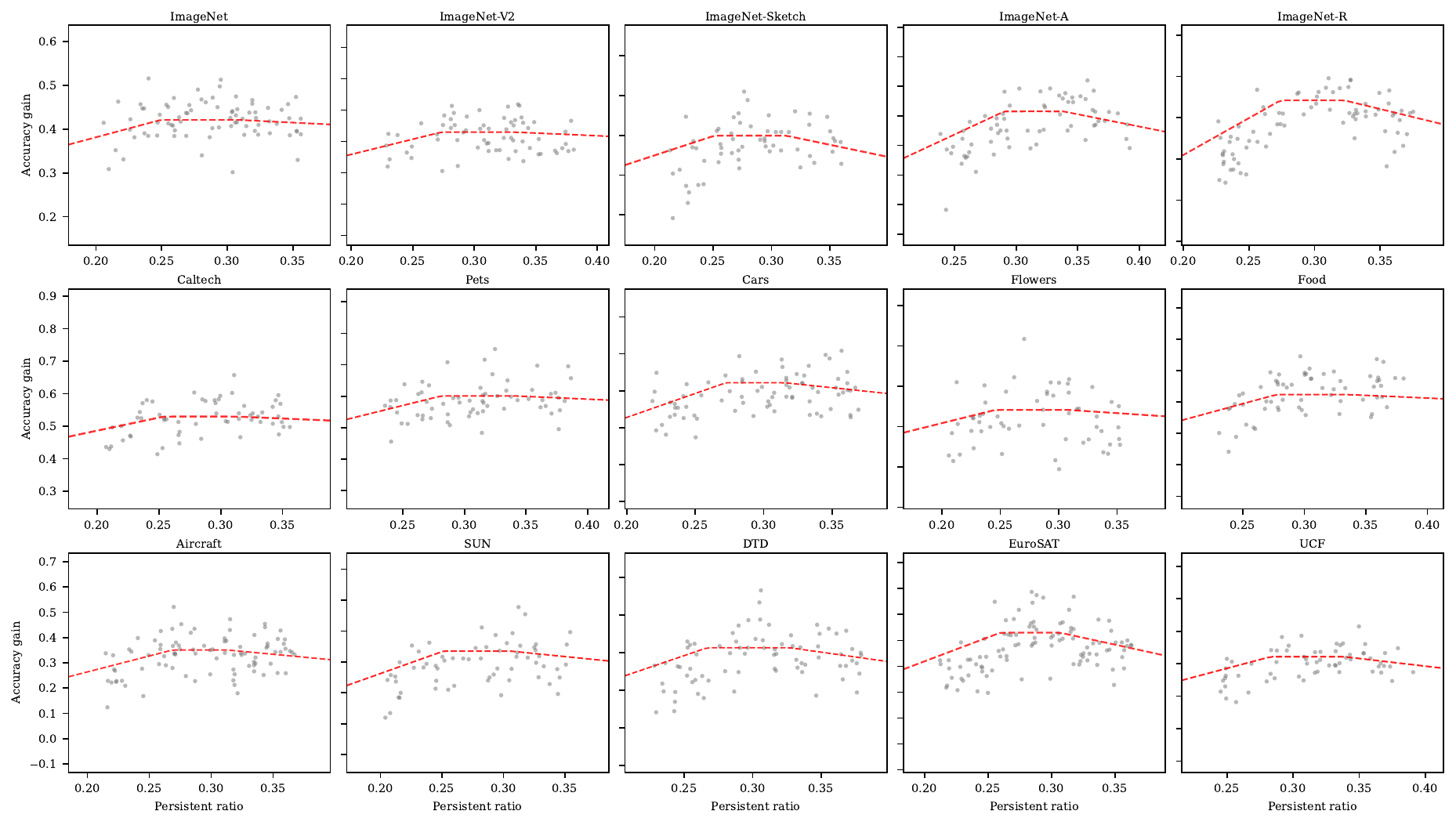}
\caption{Persistent ratio versus accuracy gain for TPT+$\mathtt{TPD}$. Each point corresponds to a local trajectory segment, and the red dashed line shows the fitted trend for each dataset.}
\label{fig:appendixA5_persistent_gain_vpt}
\end{figure*}

\subsection{Nonstationary Streams and Test-Order Variation}
\label{app:rb_selected_streams}

We evaluate whether recent update history remains useful under abrupt,
mixed, and rapidly alternating domain shifts. Each stream is evaluated
using five paired test orders shared by VPT and VPT+TPD.
Table~\ref{app:A11:streams} shows that TPD improves overall accuracy by
$1.42$, $1.53$, and $4.44$ percentage points, respectively, while reducing
the standard deviation across orders in all three settings.
Under Rapid Alternation, First-10 accuracy also increases from
$45.82\%$ to $50.09\%$, indicating improved adaptation immediately after
domain switches. These results suggest that finite-window history
remains informative under the tested distribution changes, including
a domain switch after every update.

\begingroup
\renewcommand{\thetable}{A11}
\begin{table*}[!htbp]
\caption{Accuracy (\%) under nonstationary streams. Overall reports
mean $\pm$ standard deviation across five paired test orders; the other
accuracy columns report the corresponding means.}
\label{app:A11:streams}
\label{tab:rb_selected_streams}
\begin{center}\resizebox{0.95\textwidth}{!}{%
\begin{tabular}{llcccc}
\toprule
Stream & Method & Overall & V2 & Sketch & First-10 \\
\midrule
Abrupt & VPT & $52.90\pm0.37$ & 64.52 & 29.67 & 47.28 \\
 & VPT+TPD & $54.32\pm0.32$ & 65.78 & 31.40 & 48.78 \\
\midrule
Mixed & VPT & $46.10\pm0.25$ & 65.09 & 27.11 & N/A \\
 & VPT+TPD & $47.63\pm0.19$ & 66.19 & 29.08 & N/A \\
\midrule
Rapid Alternation & VPT & $45.22\pm0.51$ & 63.82 & 26.62 & 45.82 \\
 & VPT+TPD & $49.66\pm0.27$ & 67.64 & 31.69 & 50.09 \\
\bottomrule
\end{tabular}%
}
\end{center}
\vspace{-0.5em}
\noindent
\parbox{0.95\textwidth}{\footnotesize
\textit{Setup.} Let $V$ and $K$ denote ImageNet-V2 and ImageNet-Sketch.
Abrupt: $V_{5000}\!\rightarrow\!K_{5000}\!\rightarrow\!V_{5000}$
($15{,}000$ samples).
Mixed: a global shuffle of $10{,}000$ samples from each domain
($20{,}000$ samples).
Rapid Alternation: $V_{128}\!\rightarrow\!K_{128}\!\rightarrow\!\cdots$
($20{,}000$ samples), with a domain switch after every adaptation update.
VPT and VPT+TPD use the same five predefined test orders for each stream.
First-10 measures accuracy over the first ten batches after a domain switch
and is undefined for Mixed.
}
\end{table*}
\endgroup
\FloatBarrier

\subsection{Component Ratios under Nonstationary Shifts}
\label{app:rb_selected_components}

Table~\ref{app:A12:components} examines how the decomposition changes
under nonstationary streams. For VPT, the persistent fraction decreases
from $17.54\%$ under Abrupt to $5.09\%$ under Rapid Alternation, while
the residual fraction increases from $81.75\%$ to $93.65\%$.
Nevertheless, TPD improves accuracy in every stream, with its largest
VPT gain under Rapid Alternation. TPT also benefits despite residual
fractions exceeding $96\%$. These ratios characterize local trajectory
predictability rather than semantic correctness. The results support
soft attenuation with nonzero routing weights, which allows potentially
useful corrections to remain active when updates are weakly predictable.

\begingroup
\renewcommand{\thetable}{A12}
\begin{table*}[!htbp]
\caption{Component ratios and accuracy (\%) under nonstationary shifts.
P, O, and R denote persistent, oscillatory, and residual fractions.
Accuracy is reported as Base $\rightarrow$ Base+TPD.}
\label{app:A12:components}
\label{tab:rb_selected_components}
\begin{center}\resizebox{0.95\textwidth}{!}{%
\begin{tabular}{llcccc}
\toprule
Method & Stream & P & O & R & Accuracy \\
\midrule
VPT & Abrupt & 17.54 & 0.71 & 81.75 & $52.90\rightarrow54.32$ \\
VPT & Mixed & 12.83 & 0.48 & 86.69 & $46.10\rightarrow47.63$ \\
VPT & Rapid Alternation & 5.09 & 1.26 & 93.65 & $45.22\rightarrow49.66$ \\
\midrule
TPT & Abrupt & 3.30 & 0.45 & 96.25 & $57.31\rightarrow57.92$ \\
TPT & Mixed & 3.16 & 0.44 & 96.40 & $54.85\rightarrow55.58$ \\
TPT & Rapid Alternation & 3.19 & 0.46 & 96.35 & $55.09\rightarrow55.82$ \\
\bottomrule
\end{tabular}%
}
\end{center}
\vspace{-0.5em}
\noindent
\parbox{0.95\textwidth}{\footnotesize
\textit{Note.} Streams follow Table~\ref{app:A11:streams}.
The component ratios describe local trajectory predictability under the
fitted model rather than the semantic correctness of the updates.
}
\end{table*}
\endgroup
\FloatBarrier

\subsection{Variation across Random Seeds}
\label{app:rb_selected_seeds}

We repeat VPT and TPT on ImageNet-A/R using seeds $4827$, $7316$,
and $9054$. Table~\ref{app:A13:seeds} reports the mean and standard
deviation across these independent runs. TPD improves mean accuracy
in all four settings and reduces the standard deviation in three.
For VPT on ImageNet-A, accuracy increases from
$19.25\pm1.97\%$ to $38.73\pm1.28\%$.
For TPT on ImageNet-A, the standard deviation increases slightly
from $1.12$ to $1.16$, showing that improved mean accuracy does not
uniformly imply lower run-to-run variation.

\begingroup
\renewcommand{\thetable}{A13}
\begin{table*}[!htbp]
\caption{Accuracy (\%) as mean $\pm$ standard deviation across
three random seeds ($4827$, $7316$, and $9054$). Base denotes the
corresponding VPT or TPT method without TPD.}
\label{app:A13:seeds}
\label{tab:rb_selected_seeds}
\begin{center}\resizebox{0.95\textwidth}{!}{%
\begin{tabular}{lcccc}
\toprule
 & \multicolumn{2}{c}{VPT} & \multicolumn{2}{c}{TPT} \\
\cmidrule(lr){2-3}\cmidrule(lr){4-5}
Method & ImageNet-A & ImageNet-R & ImageNet-A & ImageNet-R \\
\midrule
Base & $19.25\pm1.97$ & $42.32\pm0.61$ & $54.44\pm1.12$ & $77.01\pm2.05$ \\
Base+TPD & $38.73\pm1.28$ & $46.54\pm0.54$ & $56.62\pm1.16$ & $79.63\pm1.89$ \\
\bottomrule
\end{tabular}%
}
\end{center}
\end{table*}
\endgroup
\FloatBarrier

\subsection{Comparison with Adaptive Magnitude Controllers}
\label{app:rb_selected_controllers}

We compare TPD with Adaptive Clipping and Variance-aware Scaling,
two label-free magnitude controllers using the same base updates
and evaluation streams. Table~\ref{app:A14:controllers} shows that
TPD achieves the highest accuracy in all four settings.
For VPT on ImageNet-A, TPD reaches $38.98\%$, compared with
$23.52\%$ for Adaptive Clipping and $22.76\%$ for
Variance-aware Scaling. These comparisons complement the component
ablations by showing that the tested magnitude-only controllers
do not fully account for the gains obtained through joint magnitude
control and directional routing.

\begingroup
\renewcommand{\thetable}{A14}
\begin{table*}[!htbp]
\caption{Accuracy (\%) of adaptive magnitude controllers using the
same base updates and evaluation streams. Base denotes the
corresponding VPT or TPT method without a controller.}
\label{app:A14:controllers}
\label{tab:rb_selected_controllers}
\begin{center}\resizebox{0.95\textwidth}{!}{%
\begin{tabular}{lcccc}
\toprule
 & \multicolumn{2}{c}{VPT} & \multicolumn{2}{c}{TPT} \\
\cmidrule(lr){2-3}\cmidrule(lr){4-5}
Controller & ImageNet-A & ImageNet-R & ImageNet-A & ImageNet-R \\
\midrule
Base & 19.17 & 42.14 & 54.77 & 77.06 \\
Adaptive Clipping & 23.52 & 41.92 & 54.85 & 77.31 \\
Variance-aware Scaling & 22.76 & 43.65 & 55.19 & 77.26 \\
TPD & 38.98 & 46.47 & 56.73 & 79.98 \\
\bottomrule
\end{tabular}%
}
\end{center}
\vspace{-0.5em}
\noindent
\parbox{0.95\textwidth}{\footnotesize
\textit{Controllers.} Both comparison controllers are label-free.
Adaptive Clipping applies an adaptive magnitude cap.
Variance-aware Scaling maps the running variance of recent update norms
to a multiplicative gain.
}
\end{table*}
\endgroup
\FloatBarrier

%%%%%%%%%%%%%%%%%%%%%%%%%%%%%%%%%%%%%%%%%%%%%%%%%%%%%%%%%%%%

\clearpage
\section*{NeurIPS Paper Checklist}

\begin{enumerate}

\item {\bf Claims}
    \item[] Question: Do the main claims made in the abstract and introduction accurately reflect the paper's contributions and scope?
    \item[] Answer: \answerYes{} % Replace by \answerYes{}, \answerNo{}, or \answerNA{}.
    \item[] Justification: The abstract and introduction state claims consistent with theory and experiments.
    \item[] Guidelines:
    \begin{itemize}
        \item The answer \answerNA{} means that the abstract and introduction do not include the claims made in the paper.
        \item The abstract and/or introduction should clearly state the claims made, including the contributions made in the paper and important assumptions and limitations. A \answerNo{} or \answerNA{} answer to this question will not be perceived well by the reviewers. 
        \item The claims made should match theoretical and experimental results, and reflect how much the results can be expected to generalize to other settings. 
        \item It is fine to include aspirational goals as motivation as long as it is clear that these goals are not attained by the paper. 
    \end{itemize}

\item {\bf Limitations}
    \item[] Question: Does the paper discuss the limitations of the work performed by the authors?
    \item[] Answer: \answerYes{} % Replace by \answerYes{}, \answerNo{}, or \answerNA{}.
    \item[] Justification: The conclusion discusses limitations.
    \item[] Guidelines:
    \begin{itemize}
        \item The answer \answerNA{} means that the paper has no limitation while the answer \answerNo{} means that the paper has limitations, but those are not discussed in the paper. 
        \item The authors are encouraged to create a separate ``Limitations'' section in their paper.
        \item The paper should point out any strong assumptions and how robust the results are to violations of these assumptions (e.g., independence assumptions, noiseless settings, model well-specification, asymptotic approximations only holding locally). The authors should reflect on how these assumptions might be violated in practice and what the implications would be.
        \item The authors should reflect on the scope of the claims made, e.g., if the approach was only tested on a few datasets or with a few runs. In general, empirical results often depend on implicit assumptions, which should be articulated.
        \item The authors should reflect on the factors that influence the performance of the approach. For example, a facial recognition algorithm may perform poorly when image resolution is low or images are taken in low lighting. Or a speech-to-text system might not be used reliably to provide closed captions for online lectures because it fails to handle technical jargon.
        \item The authors should discuss the computational efficiency of the proposed algorithms and how they scale with dataset size.
        \item If applicable, the authors should discuss possible limitations of their approach to address problems of privacy and fairness.
        \item While the authors might fear that complete honesty about limitations might be used by reviewers as grounds for rejection, a worse outcome might be that reviewers discover limitations that aren't acknowledged in the paper. The authors should use their best judgment and recognize that individual actions in favor of transparency play an important role in developing norms that preserve the integrity of the community. Reviewers will be specifically instructed to not penalize honesty concerning limitations.
    \end{itemize}

\item {\bf Theory assumptions and proofs}
    \item[] Question: For each theoretical result, does the paper provide the full set of assumptions and a complete (and correct) proof?
    \item[] Answer: \answerYes{} % Replace by \answerYes{}, \answerNo{}, or \answerNA{}.
    \item[] Justification: The method and appendix provide assumptions, lemmas, theorem, and proofs.
    \item[] Guidelines:
    \begin{itemize}
        \item The answer \answerNA{} means that the paper does not include theoretical results. 
        \item All the theorems, formulas, and proofs in the paper should be numbered and cross-referenced.
        \item All assumptions should be clearly stated or referenced in the statement of any theorems.
        \item The proofs can either appear in the main paper or the supplemental material, but if they appear in the supplemental material, the authors are encouraged to provide a short proof sketch to provide intuition. 
        \item Inversely, any informal proof provided in the core of the paper should be complemented by formal proofs provided in appendix or supplemental material.
        \item Theorems and Lemmas that the proof relies upon should be properly referenced. 
    \end{itemize}

    \item {\bf Experimental result reproducibility}
    \item[] Question: Does the paper fully disclose all the information needed to reproduce the main experimental results of the paper to the extent that it affects the main claims and/or conclusions of the paper (regardless of whether the code and data are provided or not)?
    \item[] Answer: \answerYes{} % Replace by \answerYes{}, \answerNo{}, or \answerNA{}.
    \item[] Justification: The experiments and appendix specify datasets, baselines, hyperparameters, and protocols.
    \item[] Guidelines:
    \begin{itemize}
        \item The answer \answerNA{} means that the paper does not include experiments.
        \item If the paper includes experiments, a \answerNo{} answer to this question will not be perceived well by the reviewers: Making the paper reproducible is important, regardless of whether the code and data are provided or not.
        \item If the contribution is a dataset and\slash or model, the authors should describe the steps taken to make their results reproducible or verifiable. 
        \item Depending on the contribution, reproducibility can be accomplished in various ways. For example, if the contribution is a novel architecture, describing the architecture fully might suffice, or if the contribution is a specific model and empirical evaluation, it may be necessary to either make it possible for others to replicate the model with the same dataset, or provide access to the model. In general. releasing code and data is often one good way to accomplish this, but reproducibility can also be provided via detailed instructions for how to replicate the results, access to a hosted model (e.g., in the case of a large language model), releasing of a model checkpoint, or other means that are appropriate to the research performed.
        \item While NeurIPS does not require releasing code, the conference does require all submissions to provide some reasonable avenue for reproducibility, which may depend on the nature of the contribution. For example
        \begin{enumerate}
            \item If the contribution is primarily a new algorithm, the paper should make it clear how to reproduce that algorithm.
            \item If the contribution is primarily a new model architecture, the paper should describe the architecture clearly and fully.
            \item If the contribution is a new model (e.g., a large language model), then there should either be a way to access this model for reproducing the results or a way to reproduce the model (e.g., with an open-source dataset or instructions for how to construct the dataset).
            \item We recognize that reproducibility may be tricky in some cases, in which case authors are welcome to describe the particular way they provide for reproducibility. In the case of closed-source models, it may be that access to the model is limited in some way (e.g., to registered users), but it should be possible for other researchers to have some path to reproducing or verifying the results.
        \end{enumerate}
    \end{itemize}

\item {\bf Open access to data and code}
    \item[] Question: Does the paper provide open access to the data and code, with sufficient instructions to faithfully reproduce the main experimental results, as described in supplemental material?
    \item[] Answer: \answerYes{}  % Replace by \answerYes{}, \answerNo{}, or \answerNA{}.
    \item[] Justification: The code is available at \href{https://github.com/reeive/TPD}{https://github.com/reeive/TPD}.
    \item[] Guidelines:
    \begin{itemize}
        \item The answer \answerNA{} means that paper does not include experiments requiring code.
        \item Please see the NeurIPS code and data submission guidelines (\url{https://neurips.cc/public/guides/CodeSubmissionPolicy}) for more details.
        \item While we encourage the release of code and data, we understand that this might not be possible, so \answerNo{} is an acceptable answer. Papers cannot be rejected simply for not including code, unless this is central to the contribution (e.g., for a new open-source benchmark).
        \item The instructions should contain the exact command and environment needed to run to reproduce the results. See the NeurIPS code and data submission guidelines (\url{https://neurips.cc/public/guides/CodeSubmissionPolicy}) for more details.
        \item The authors should provide instructions on data access and preparation, including how to access the raw data, preprocessed data, intermediate data, and generated data, etc.
        \item The authors should provide scripts to reproduce all experimental results for the new proposed method and baselines. If only a subset of experiments are reproducible, they should state which ones are omitted from the script and why.
        \item At submission time, to preserve anonymity, the authors should release anonymized versions (if applicable).
        \item Providing as much information as possible in supplemental material (appended to the paper) is recommended, but including URLs to data and code is permitted.
    \end{itemize}

\item {\bf Experimental setting/details}
    \item[] Question: Does the paper specify all the training and test details (e.g., data splits, hyperparameters, how they were chosen, type of optimizer) necessary to understand the results?
    \item[] Answer: \answerYes{} % Replace by \answerYes{}, \answerNo{}, or \answerNA{}.
    \item[] Justification: The experimental setup details tasks, datasets, settings, baselines, and hyperparameters.
    \item[] Guidelines:
    \begin{itemize}
        \item The answer \answerNA{} means that the paper does not include experiments.
        \item The experimental setting should be presented in the core of the paper to a level of detail that is necessary to appreciate the results and make sense of them.
        \item The full details can be provided either with the code, in appendix, or as supplemental material.
    \end{itemize}

\item {\bf Experiment statistical significance}
    \item[] Question: Does the paper report error bars suitably and correctly defined or other appropriate information about the statistical significance of the experiments?
    \item[] Answer: \answerYes{} % Replace by \answerYes{}, \answerNo{}, or \answerNA{}.
    \item[] Justification: The appendix reports extensive per-dataset, sensitivity, and ablation analyses supporting robustness across settings.
    \item[] Guidelines:
    \begin{itemize}
        \item The answer \answerNA{} means that the paper does not include experiments.
        \item The authors should answer \answerYes{} if the results are accompanied by error bars, confidence intervals, or statistical significance tests, at least for the experiments that support the main claims of the paper.
        \item The factors of variability that the error bars are capturing should be clearly stated (for example, train/test split, initialization, random drawing of some parameter, or overall run with given experimental conditions).
        \item The method for calculating the error bars should be explained (closed form formula, call to a library function, bootstrap, etc.)
        \item The assumptions made should be given (e.g., Normally distributed errors).
        \item It should be clear whether the error bar is the standard deviation or the standard error of the mean.
        \item It is OK to report 1-sigma error bars, but one should state it. The authors should preferably report a 2-sigma error bar than state that they have a 96\% CI, if the hypothesis of Normality of errors is not verified.
        \item For asymmetric distributions, the authors should be careful not to show in tables or figures symmetric error bars that would yield results that are out of range (e.g., negative error rates).
        \item If error bars are reported in tables or plots, the authors should explain in the text how they were calculated and reference the corresponding figures or tables in the text.
    \end{itemize}

\item {\bf Experiments compute resources}
    \item[] Question: For each experiment, does the paper provide sufficient information on the computer resources (type of compute workers, memory, time of execution) needed to reproduce the experiments?
    \item[] Answer: \answerYes{} % Replace by \answerYes{}, \answerNo{}, or \answerNA{}.
    \item[] Justification: The paper reports H100 GPU use and wall-clock overhead.
    \item[] Guidelines:
    \begin{itemize}
        \item The answer \answerNA{} means that the paper does not include experiments.
        \item The paper should indicate the type of compute workers CPU or GPU, internal cluster, or cloud provider, including relevant memory and storage.
        \item The paper should provide the amount of compute required for each of the individual experimental runs as well as estimate the total compute. 
        \item The paper should disclose whether the full research project required more compute than the experiments reported in the paper (e.g., preliminary or failed experiments that didn't make it into the paper). 
    \end{itemize}
    
\item {\bf Code of ethics}
    \item[] Question: Does the research conducted in the paper conform, in every respect, with the NeurIPS Code of Ethics \url{https://neurips.cc/public/EthicsGuidelines}?
    \item[] Answer: \answerYes{} % Replace by \answerYes{}, \answerNo{}, or \answerNA{}.
    \item[] Justification: The experiments use public benchmarks and comply with NeurIPS ethics.
    \item[] Guidelines:
    \begin{itemize}
        \item The answer \answerNA{} means that the authors have not reviewed the NeurIPS Code of Ethics.
        \item If the authors answer \answerNo, they should explain the special circumstances that require a deviation from the Code of Ethics.
        \item The authors should make sure to preserve anonymity (e.g., if there is a special consideration due to laws or regulations in their jurisdiction).
    \end{itemize}

\item {\bf Broader impacts}
    \item[] Question: Does the paper discuss both potential positive societal impacts and negative societal impacts of the work performed?
    \item[] Answer: \answerNA{} % Replace by \answerYes{}, \answerNo{}, or \answerNA{}.
    \item[] Justification: The paper focuses on methodology and does not separately discuss broader impacts.
    \item[] Guidelines:
    \begin{itemize}
        \item The answer \answerNA{} means that there is no societal impact of the work performed.
        \item If the authors answer \answerNA{} or \answerNo, they should explain why their work has no societal impact or why the paper does not address societal impact.
        \item Examples of negative societal impacts include potential malicious or unintended uses (e.g., disinformation, generating fake profiles, surveillance), fairness considerations (e.g., deployment of technologies that could make decisions that unfairly impact specific groups), privacy considerations, and security considerations.
        \item The conference expects that many papers will be foundational research and not tied to particular applications, let alone deployments. However, if there is a direct path to any negative applications, the authors should point it out. For example, it is legitimate to point out that an improvement in the quality of generative models could be used to generate Deepfakes for disinformation. On the other hand, it is not needed to point out that a generic algorithm for optimizing neural networks could enable people to train models that generate Deepfakes faster.
        \item The authors should consider possible harms that could arise when the technology is being used as intended and functioning correctly, harms that could arise when the technology is being used as intended but gives incorrect results, and harms following from (intentional or unintentional) misuse of the technology.
        \item If there are negative societal impacts, the authors could also discuss possible mitigation strategies (e.g., gated release of models, providing defenses in addition to attacks, mechanisms for monitoring misuse, mechanisms to monitor how a system learns from feedback over time, improving the efficiency and accessibility of ML).
    \end{itemize}
    
\item {\bf Safeguards}
    \item[] Question: Does the paper describe safeguards that have been put in place for responsible release of data or models that have a high risk for misuse (e.g., pre-trained language models, image generators, or scraped datasets)?
    \item[] Answer: \answerNA{} % Replace by \answerYes{}, \answerNo{}, or \answerNA{}.
    \item[] Justification: The paper poses no such risks.
    \item[] Guidelines:
    \begin{itemize}
        \item The answer \answerNA{} means that the paper poses no such risks.
        \item Released models that have a high risk for misuse or dual-use should be released with necessary safeguards to allow for controlled use of the model, for example by requiring that users adhere to usage guidelines or restrictions to access the model or implementing safety filters. 
        \item Datasets that have been scraped from the Internet could pose safety risks. The authors should describe how they avoided releasing unsafe images.
        \item We recognize that providing effective safeguards is challenging, and many papers do not require this, but we encourage authors to take this into account and make a best faith effort.
    \end{itemize}

\item {\bf Licenses for existing assets}
    \item[] Question: Are the creators or original owners of assets (e.g., code, data, models), used in the paper, properly credited and are the license and terms of use explicitly mentioned and properly respected?
    \item[] Answer: \answerYes{} % Replace by \answerYes{}, \answerNo{}, or \answerNA{}.
    \item[] Justification: All used assets (public datasets, algorithm implementations) are cited.
    \item[] Guidelines:
    \begin{itemize}
        \item The answer \answerNA{} means that the paper does not use existing assets.
        \item The authors should cite the original paper that produced the code package or dataset.
        \item The authors should state which version of the asset is used and, if possible, include a URL.
        \item The name of the license (e.g., CC-BY 4.0) should be included for each asset.
        \item For scraped data from a particular source (e.g., website), the copyright and terms of service of that source should be provided.
        \item If assets are released, the license, copyright information, and terms of use in the package should be provided. For popular datasets, \url{paperswithcode.com/datasets} has curated licenses for some datasets. Their licensing guide can help determine the license of a dataset.
        \item For existing datasets that are re-packaged, both the original license and the license of the derived asset (if it has changed) should be provided.
        \item If this information is not available online, the authors are encouraged to reach out to the asset's creators.
    \end{itemize}

\item {\bf New assets}
    \item[] Question: Are new assets introduced in the paper well documented and is the documentation provided alongside the assets?
    \item[] Answer: \answerYes{} % Replace by \answerYes{}, \answerNo{}, or \answerNA{}.
    \item[] Justification: The code is available at \href{https://github.com/reeive/TPD}{https://github.com/reeive/TPD}.
    \item[] Guidelines:
    \begin{itemize}
        \item The answer \answerNA{} means that the paper does not release new assets.
        \item Researchers should communicate the details of the dataset\slash code\slash model as part of their submissions via structured templates. This includes details about training, license, limitations, etc. 
        \item The paper should discuss whether and how consent was obtained from people whose asset is used.
        \item At submission time, remember to anonymize your assets (if applicable). You can either create an anonymized URL or include an anonymized zip file.
    \end{itemize}

\item {\bf Crowdsourcing and research with human subjects}
    \item[] Question: For crowdsourcing experiments and research with human subjects, does the paper include the full text of instructions given to participants and screenshots, if applicable, as well as details about compensation (if any)? 
    \item[] Answer: \answerNA{} % Replace by \answerYes{}, \answerNo{}, or \answerNA{}.
    \item[] Justification: The paper does not involve crowdsourcing nor research with human subjects.
    \item[] Guidelines:
    \begin{itemize}
        \item The answer \answerNA{} means that the paper does not involve crowdsourcing nor research with human subjects.
        \item Including this information in the supplemental material is fine, but if the main contribution of the paper involves human subjects, then as much detail as possible should be included in the main paper. 
        \item According to the NeurIPS Code of Ethics, workers involved in data collection, curation, or other labor should be paid at least the minimum wage in the country of the data collector. 
    \end{itemize}

\item {\bf Institutional review board (IRB) approvals or equivalent for research with human subjects}
    \item[] Question: Does the paper describe potential risks incurred by study participants, whether such risks were disclosed to the subjects, and whether Institutional Review Board (IRB) approvals (or an equivalent approval/review based on the requirements of your country or institution) were obtained?
    \item[] Answer: \answerNA{} % Replace by \answerYes{}, \answerNo{}, or \answerNA{}.
    \item[] Justification: The paper does not involve crowdsourcing nor research with human subjects.
    \item[] Guidelines:
    \begin{itemize}
        \item The answer \answerNA{} means that the paper does not involve crowdsourcing nor research with human subjects.
        \item Depending on the country in which research is conducted, IRB approval (or equivalent) may be required for any human subjects research. If you obtained IRB approval, you should clearly state this in the paper. 
        \item We recognize that the procedures for this may vary significantly between institutions and locations, and we expect authors to adhere to the NeurIPS Code of Ethics and the guidelines for their institution. 
        \item For initial submissions, do not include any information that would break anonymity (if applicable), such as the institution conducting the review.
    \end{itemize}

\item {\bf Declaration of LLM usage}
    \item[] Question: Does the paper describe the usage of LLMs if it is an important, original, or non-standard component of the core methods in this research? Note that if the LLM is used only for writing, editing, or formatting purposes and does \emph{not} impact the core methodology, scientific rigor, or originality of the research, declaration is not required.
    %this research? 
    \item[] Answer: \answerNA{} % Replace by \answerYes{}, \answerNo{}, or \answerNA{}.
    \item[] Justification: The core method and experiments do not use LLMs.
    \item[] Guidelines:
    \begin{itemize}
        \item The answer \answerNA{} means that the core method development in this research does not involve LLMs as any important, original, or non-standard components.
        \item Please refer to our LLM policy in the NeurIPS handbook for what should or should not be described.
    \end{itemize}

\end{enumerate}

\end{document}